\documentclass[final,onefignum,onetabnum]{siamonline190516}

\graphicspath{{Figs_video/}}

\usepackage{hyperref}

\usepackage{soul}
\usepackage{bbm}
\usepackage{amsmath}
\usepackage{comment}

\usepackage{listings}
\usepackage{color}
\usepackage{xcolor}

\usepackage{tablefootnote}
\usepackage{dsfont}
\usepackage{amssymb}
\usepackage{enumerate}
\usepackage{multirow}
\usepackage{hhline}
\usepackage{color}
\usepackage{epsfig}
\usepackage{subfig}
\usepackage{graphicx}
\usepackage{mathrsfs}
\usepackage{float}
\usepackage{bbm,bm}
\usepackage{multirow}
\usepackage[shortlabels]{enumitem}

\usepackage[algo2e]{algorithm2e}
\usepackage{algorithmic}

\usepackage{multicol}
\usepackage{soul}
\usepackage{epsfig}
\allowdisplaybreaks

\hypersetup{
	colorlinks,
	citecolor=blue,
	filecolor=blue,
	linkcolor=blue,
	urlcolor=blue,
}

\newsiamremark{remark}{Remark}

\newcommand{\be}{\bm{e}}

\newcommand{\R}{\mathbb{R}}

\newcommand{\N}{\mathbb{N}}

\newcommand{\cO}{\mathcal{O}}

\newcommand{\BL}{\bm{L}}
\newcommand{\BS}{\bm{S}}
\newcommand{\BD}{\bm{D}}
\newcommand{\BC}{\bm{C}}
\newcommand{\BR}{\bm{R}}
\newcommand{\BW}{\bm{W}}
\newcommand{\BU}{\bm{U}}
\newcommand{\BV}{\bm{V}}
\newcommand{\BE}{\bm{E}}
\newcommand{\BX}{\bm{X}}
\newcommand{\BSigma}{\bm{\Sigma}}

\newcommand{\supp}{{\rm supp\,}}
\newcommand{\rank}{\mathrm{rank}}

\DeclareMathOperator*{\argmin}{\mathrm{argmin}}

\newcommand{\eps}{\varepsilon}
\newcommand{\RV}[1]{\textcolor{black}{#1}}

\title{Robust CUR Decomposition: Theory and Imaging Applications\thanks{\funding{K.~Hamm was sponsored in part by the Army Research Office under grant number W911NF-20-1-0076. The views and conclusions contained in this document are those of the authors and should not be interpreted as representing the official policies, either expressed or implied, of the Army Research Office or the U.S. Government. The U.S. Government is authorized to reproduce and distribute reprints for Government purposes notwithstanding any copyright notation herein. L.~Huang and D.~Needell were supported by NSF BIGDATA DMS \#1740325 and NSF DMS \#2011140.}}} 

\author{HanQin Cai\thanks{Department of Mathematics, 
University of California, Los Angeles, Los Angeles, CA (\email{hqcai@math.ucla.edu}).}  \and Keaton Hamm\thanks{Department of Mathematics, University of Texas at Arlington, Arlington, TX
({\color{red}\email{keaton.hamm@uta.edu}}).}  \and Longxiu Huang\thanks{Department of Mathematics, 
University of California, Los Angeles, Los Angeles, CA (\email{huangl3@math.ucla.edu}).} \and
Deanna Needell\thanks{Department of Mathematics, 
University of California, Los Angeles, Los Angeles, CA (\email{deanna@math.ucla.edu}).}
}
\begin{document}

\maketitle

\begin{abstract}
This paper considers the use of Robust PCA in a CUR decomposition framework and applications thereof. Our main algorithms produce a robust version of column-row factorizations of matrices $\BD=\BL+\BS$ where $\BL$ is low-rank and $\BS$ contains sparse outliers. These methods yield interpretable factorizations at low computational cost, and provide new CUR decompositions that are robust to sparse outliers, in contrast to previous methods. We consider two key imaging applications of Robust PCA: video foreground-background separation and face modeling. This paper examines the qualitative behavior of our Robust CUR decompositions on the benchmark videos and face datasets, and 
find that our method works as well as standard Robust PCA while being significantly faster. Additionally, we consider hybrid randomized and deterministic sampling methods which produce a compact CUR decomposition of a given matrix, and apply this to video sequences to produce canonical frames thereof.
\end{abstract}

\section{Introduction}
The scale of data being collected is increasingly massive, which requires better and more computationally efficient tools to handle associated data matrices.  One method of handling large data matrices is to extract a meaningful column submatrix, and to either represent the full data matrix via a factorization with respect to the columns or to  carry out any desired analysis on the column submatrix, and then reconstruct the full data matrix from the columns.  This technique is motivated by some foundational results that tell one when and how well column submatrices can represent a data matrix.  In particular, for a low-rank matrix $\BL\in \mathbb{R}^{m\times n}$ with $\rank(\BL)=r$, one can write $\BL=\BC\BX$ where $\BC$ is a (much smaller) column submatrix of $\BL$ as long as $\rank(\BC)=\rank(\BL)$.  Hence, one obtains a representation of $\BL$ in terms of a few representative data vectors. 

In general, such factorizations are called \textit{interpolative decompositions} \cite{VoroninMartinsson}, while for particular choices of $\BX$ including $\BX=\BC^\dagger\BL$ or $\BX=\BU^\dagger \BR$ where $\BR$ is a row submatrix of $\BL$, and $\BU$ is the overlap of the column and row indices of $\BC$ and $\BR$, they are termed \textit{CUR decompositions} \cite{BoutsidisOptimalCUR,DKMIII,DMM08,HammHuang,DMPNAS,SorensenDEIMCUR}. By some, the latter examples are called \textit{cross-approximations} \cite{oseledets2010tt} or \textit{(pseudo-)skeleton approximations} \cite{DemanetWu,Goreinov2,Goreinov,Goreinov3}. The problem of finding the best set of $k$ columns with which to approximate $\BL$ is the Column Subset Selection Problem \cite{BoutsidisNearOptimalCSSP}.  

Since their advent, column factorizations such as these have been shown to have great flexibility and approximation power as a general low-rank approximation tool. Specifically, they provide advantages over traditional factorization methods such as: 1) column factorizations can yield compact representations of large data matrices, 2) for incoherent matrices, the complexity of generating the uniform sampling pattern is $\cO(1)$\footnote{Here and throughout, we take $\cO(1)$ to mean asymptotically with respect to $m$ and $n$.} while yielding a factorization which well-approximates the SVD (which requires $\cO(mnr)$ time for a rank $r$ matrix), and 3) the representation of a matrix in a basis consisting of actual data vectors provides a more interpretable representation of the data than abstract factorizations such as QR or SVD.

In many image and video processing applications, the observed data can be interpreted as:
\begin{enumerate}
    \item Low rank information of interest with some sparse outliers. For example, in the task of face modeling \cite{wright2008robust}, the face data of the same person in different photos form a low rank matrix that we want to recover while the facial occlusions (accessories, 
    shadows and facial expressions) are considered sparse outliers.
    \item A combination of two pieces of information of interest that have low rank and sparse properties, respectively. For example, in the task of video background subtraction \cite{jang2016primary}, the static background has the low rank property and the foreground (i.e., the moving objects) is sparse.   
\end{enumerate}
Robust PCA (RPCA) is an apt tool for solving the aforementioned problems and various other tasks in this arena \cite{bouwmans2018applications,candes2011robust,xu2012robust}.  

In the Robust PCA framework, we assume we observe $\BD=\BL+\BS \in\R^{m\times n}$ where $\BL$ is a low-rank matrix and $\BS$ is a sparse matrix, and the goal is to recover $\BL$.  For large-scale data problems, $\BD$ may not be able to be stored in memory.  The purpose of this paper is to elucidate how one may use column factorizations to assist in the Robust PCA task of finding the low-rank piece. We elaborate on our main contributions in Section~\ref{SEC:Contributions} below, but first we present the main ideas utilized throughout the paper.

\subsection{Main Ideas}

There are several points to consider when trying to marry the techniques of CUR decompositions and robust low-rank matrix recovery from sparse outliers.  The main goal of this paper is to do so in such a way as to guarantee speed, interpretability, and robustness. The first two goals are commonly achieved by CUR decompositions (as opposed to SVD based methods) , while the third is achieved by Robust PCA but not typically by CUR.  We elaborate more on how the method proposed here achieves these aims in comparison with existing techniques in the following subsection.

Taking these factors into consideration, we propose the following broadly applicable method: subsample columns and rows of $\BD$, apply an existing Robust PCA algorithm to these matrices, and use the denoised column and row matrices to form a CUR decomposition of the underlying low-rank part, $\BL$.  This procedure has several advantages: 1) it is faster than standard Robust PCA as long as few columns and rows may be selected, which is guaranteed by some theoretical results on submatrices of incoherent matrices, 2) it can work on matrices that do not fit in memory because it first subsamples the data matrix, 3) it returns a column-row factorization that is known to be more interpretable than the SVD as it corresponds to representing data via other actual data points, 4) there are good theoretical guarantees for its performance, and 5) it is a flexible framework that allows the user to specify the Robust PCA algorithm utilized.  The last point allows one to determine the tradeoff between robustness and computational cost existing in the various Robust PCA methods in use.

\subsection{Main Contributions}\label{SEC:Contributions}

We consider two variations of algorithms for producing Robust CUR decompositions.  In particular, given $\BD=\BL+\BS$, our first algorithm (Algorithm~\ref{ALG:Uniform}) uniformly randomly selects a column submatrix $\BC=\BD(:,J)$ and row submatrix $\BR=\BD(I,:)$ of $\BD$, then performs Robust PCA on both of these matrices to produce $\widehat\BC$ and $\widehat\BR$ which approximate $\BL(:,J)$ and $\BL(I,:)$, respectively.  Then we obtain an approximate CUR decomposition of the underlying low-rank matrix $\BL$; namely, $\BL\approx \widehat\BC(\widehat{\BC}(I,:))^\dagger\widehat\BR \approx \BL(:,J)(\BL(I,J))^\dagger\BL(I,:).$  We call such decompositions Robust CUR (RCUR) decompositions.

The theoretical contributions of this procedure are as follows:
\begin{enumerate}
    \item We prove bounds for all of the relevant properties of Robust PCA algorithms (incoherence, sparsity, and condition number) of randomly selected column and row submatrices of $\BL$ in terms of the corresponding properties of $\BL$.
    \item Our algorithms provide a novel robust version of CUR decompositions.
    \item Our error analysis shows that Robust CUR decompositions are capable of well approximating low-rank matrices with sparse, but \RV{possibly} large magnitude noise, \RV{a task for which standard CUR decompositions are not designed.}
    \item Our Robust CUR decomposition yields relative error spectral norm guarantees, which have not been shown previously for standard CUR decompositions.
    \item The column/row subsampling procedure allows our algorithm to work on large data matrices which cannot be stored in memory, and typically has less computational cost than Robust PCA on the whole data matrix.
    \item Our algorithm is flexible, as it allows the user to specify the Robust PCA algorithm used, and thus to specify the tradeoff between robustness and computational cost.
\end{enumerate}

For independent interest, we also consider how properties such as incoherence pass to submatrices under the greedy column selection procedure of \cite{AB2013}.  Under greedy column selection, one may have adversarial deterministic sparse noise which prevents the column submatrix from being sparse. 

Our second main algorithm is a hybrid random $+$ deterministic method for choosing exactly $\rank(\BL)$ columns and rows of $\BD$ to produce the most compact CUR decomposition of $\BL$. We show that with high probability, this algorithm produces a good CUR approximation of $\BL$ in the spectral norm.

\subsection{Prior Art}

Both CUR decompositions and Robust PCA algorithms have been well developed in the literature. However, our work is one of the first to combine them.  Randomized sampling algorithms for CUR decompositions and the column subset selection problem have been studied in \cite{DemanetWu,DKMIII,DMM08,HH_PCD2019,DMPNAS,WZ_2013,tropp2009column}, for example, while deterministic sampling methods have been explored in \cite{AB2013,AltschulerGreedyCSSP,LiDeterministicCSSP}, with hybrid methods explored in \cite{BoutsidisNearOptimalCSSP,BoutsidisOptimalCUR}. For a general overview, of CUR decompositions and approximations, see \cite{HammHuang}.  In contrast to the current paper, CUR decompositions typically exhibit error guarantees in the Frobenius norm, while ours hold for the spectral norm. 
\RV{Additionally, typical upper bounds for CUR decompositions of matrices hold in the worst case when the matrix is full rank, and the CUR decomposition approximates the truncated SVD. Thus CUR decompositions are primarily a low-rank approximation tool. This is in contrast to Robust PCA algorithms, which are primarily concerned with estimating a low rank matrix which is corrupted by sparse outliers, i.e., RPCA is a low-rank recovery tool.}

The algorithms developed herein apply Robust PCA in a CUR decomposition framework \cite{cai2019accelerated,candes2011robust,netrapalli2014non}.  Our Algorithm~\ref{ALG:Uniform} is faster than merely applying a Robust PCA algorithm directly to $\BD$, and is usable for matrices which are too large to store in memory; however, this comes at a small cost in robustness to extreme outliers.  We note that the authors in \cite{cai2020rapid} used CUR decompositions as an approximation to the SVD within an iterative RPCA framework; however, while that algorithm achieved empirical success, there is currently no theoretical result for the algorithm convergence or outlier toleration.

\section{Preliminaries and Layout}

We will always assume that $\BL=\bm {\BW}_{\BL}\bm{\BSigma}_{\BL} {\BV}_{\BL}^T$ is the compact SVD of $\BL\in\mathbb{R}^{m\times n}$. The symbol $[n]$ \RV{denotes the set of integers} $\{1,\dots,n\}$ for any $n\in\N$.  We use $\kappa(\BL):=\|\BL\|_2\|\BL^\dagger\|_2 = \frac{\sigma_{\max}(\BL)}{\sigma_{\min}(\BL)}$ to denote the spectral condition number of $\BL$, where $\sigma_{\max}(\BL)$ and $\sigma_{\min}(\BL)$ are the maximum and minimum non-zero singular values of $\BL$ respectively.

Below are two fundamental assumptions which are typically made in the Robust PCA framework, which we also utilize.

\begin{definition}
Let $\BL\in\R^{m\times n}$ with $\rank(\BL)=r$, and let $\BL={\BW}_{\BL}\bm{\BSigma}_{\BL}{\BV}_{\BL}^T$ be its compact SVD.  Then $\BL$ has $\left\{\mu_1(\BL),\mu_2(\BL)\right\}$--incoherence (i.e., $\mu_1(\BL)$--column incoherence and $\mu_2(\BL)$--row incoherence) for some constants $1\leq \mu_1(\BL)\leq \frac{m}{r}$, $1\leq\mu_2(\BL)\leq\frac{n}{r}$ provided
\begin{equation}\label{EQN:A1}
    \max_{i}\left\| {\BW}_{\BL}^T\bm{e}_i\right\|_2\leq\sqrt{\frac{\mu_1(\BL) r}{m}} \quad \textnormal{and} \quad \max_{i}\left\|\bm{{\BV}}_{\BL}^T\bm{e}_i\right\|_2\leq \sqrt{\frac{\mu_2(\BL) r}{n}}. \tag{A1}
\end{equation}
\end{definition}

\begin{definition}
Let $\BS\in \mathbb{R}^{m\times n}$ be an $\alpha$--sparse matrix and let $\bm{e}_i$ be the $i$--th canonical unit basis vector of the appropriate dimension. Then $\BS$ has at most $\alpha n$  non-zero entries in each row, and at most $\alpha m$ non-zero entries in each column.  In the other words,
\begin{equation}\label{EQN:A2}
\max_{i}\left\|\BS^T\bm{e}_i\right\|_0 \leq \alpha n\quad\textnormal{and} \quad \max_{j}\left\|\BS\bm{e}_j\right\|_0 \leq \alpha m \tag{A2}
\end{equation}
where $\left\|\,\cdot\,\right\|_0$ denotes the $\ell_0$-norm.
\end{definition}

\begin{remark}
    For a successful reconstruction, we expect $\mu_{i}(\BL)=\cO(1)$, $i=1,2$, which is generally true in real-world RPCA applications \cite{bouwmans2018applications,cai2021accelerated,cai2019accelerated}. For the purpose of theoretical analysis, if $\mu_{i}(\BL)$ is too large, then the RPCA problem may become too difficult to solve since the toleration of $\alpha$ usually depends on $\mathrm{poly}(\frac{1}{\mu_{i}(\BL)})$ \cite{cai2019accelerated,candes2011robust,netrapalli2014non,yi2016fast}. Additionally, the rank-sparsity uncertainty principle \cite{chandrasekaran2011rank} states that a matrix cannot be simultaneously incoherent and sparse; that is, small $\mu_i(\BL)$ ensures $\BL$ is not sparse.
\end{remark}

\subsection{Organization} The rest of the paper is laid out as follows: Section~\ref{SEC:Main} contains statements of our main results as well as our main algorithm; Section~\ref{SEC:Submatrices} contains the proof of our results concerning how incoherence, sparsity, and condition number change via column selection (Theorem~\ref{THM:BetaBound}). Section~\ref{SEC:Simulations} illustrates our methods via experiments on video background-foreground separation tasks.  Finally, the technical details of proofs of the main results are contained in the Appendices; specifically, the proof of Theorem~\ref{COR:BetaDeltaBound} is contained in Appendix~\ref{APP:ProofBeta}, the combined error analysis of CUR decompositions and Robust PCA leading to the proof of Theorem~\ref{THM:MainErrorEstimate} is in Appendix~\ref{APP:ProofMain}, the proofs for the combined random and deterministic sampling result of Theorem~\ref{THM:Hybrid} is in Appendix~\ref{APP:Hybrid}, while the purely deterministic sampling results are contained in Appendices~\ref{APP:Det1} and \ref{APP:Det2}.

\section{Main Results}\label{SEC:Main}

To develop a framework of combining Robust PCA and CUR decompositions, one must first understand how properties of incoherence and sparsity are inherited by submatrices.  This, in turn, requires estimating pseudoinverses of submatrices of singular vectors of $\BL$, which is generally a difficult task. 

\RV{Recall the first step in the Robust CUR decomposition described above: we select column and row submatrices of the observed data $\BD=\BL+\BS$, $\widetilde{\BC}=\BD(:,J)$ and $\widetilde{\BR}=\BD(I,:)$.  We then apply Robust PCA algorithms to $\widetilde{\BC}$ and $\widetilde{\BR}$ to find ``good'' approximations for the submatrices $\BC=\BL(:,J)$ and $\BR=\BL(I,:)$ of the underlying low rank matrix $\BL$. To apply the Robust PCA algorithms on $\widetilde{\BC}$ and $\widetilde{\BR}$, we must  ensure that $\BC$ and $\BR$ are incoherent, and that $\BS(I,:)$, $\BS(:.J)$ are sparse. In the next subsection, we provide quite general bounds for the incoherence of submatrices for various column selection methods.}

\RV{Second, to carry out an error analysis for Robust CUR decompositions, we must consider a concrete sampling method for which the submatrices $\BC$ and $\BR$ of $\BL$ have low incoherence, the column and row submatrices of $\BS$ are sparse, and consider the error of the CUR decomposition $\BL- \widehat{\BC}(\widehat{\BC}(I,:))^\dagger\widehat{\BR}$ where $\widehat{\BC}$ is the output of a given RPCA algorithm whose input is $\widetilde{\BC}$.  The main results pertaining to this step are in Section \ref{SEC:Main-RPCA}.
} 

\subsection{How Incoherence Passes to Submatrices}

We begin with our first main result that gives generic bounds on incoherence and condition numbers of column submatrices of low-rank matrices. \RV{Since this matter is of general interest, in this subsection, we will suppose that we have access to a low-rank matrix $\BL$ itself, which is not true in the RPCA problem. We stress that this assumption is not at all necessary, but we will show that one can obtain concrete bounds for incoherence of submatrices compared to incoherence of the full matrix for a variety of column sampling schemes.}

\begin{theorem}\label{THM:BetaBound}
Suppose $\BL\in\R^{m\times n}$ has rank $r$ and satisfies \eqref{EQN:A1}.  Let $J\subseteq[n]$ such that $\BC=\BL(:,J)$ has rank $r$, and denote $\beta:=\sqrt{\frac{|J|}{n}}\left\|{\BV}_{\BL}(J,:)^\dagger\right\|_2$. Then
\begin{enumerate}[leftmargin=1cm]
    \item $\max_i\left\| \bm {\BW}_{\BC}^T\bm{e}_i\right\|_2\leq \sqrt{\frac{\mu_1(\BL) r}{m}}$,\\
    \item $\max_i\left\| {\BV}_{\BC}^T\bm{e}_i\right\|_2\leq \beta\kappa(\BL)\sqrt{\frac{\mu_2(\BL) r}{| J|}}$,\\  
    \item $\kappa(\BC)\leq\beta\sqrt{\mu_2(\BL) r}\kappa(\BL)$.
\end{enumerate}
In particular, \[\mu_1(\BC)\leq\mu_1(\BL),\quad \textnormal{and}\quad \mu_2(\BC)\leq\beta^2\kappa(\BL)^2\mu_2(\BL).\]
\end{theorem}

\begin{remark}
We assume the condition number $\kappa(\BL)=\cO(1)$, which is commonly assumed in many RPCA papers \cite{cai2021accelerated,cai2019accelerated,yi2016fast}. According to Theorem~\ref{THM:BetaBound}, a well-conditioned $\BL$ with good incoherence ensures the corresponding properties of its submatrices, and this is a key to the success of RCUR. 
\end{remark}

Naturally, the value of $\beta$ will depend heavily on the sampling method utilized to select the column indices $J$.  Thus, we desire algorithms for selecting column submatrices of $\BL$ for which $\beta\ll \cO(n)$.  However, it is typically the case that $|J|$ and $\beta$ are inversely proportional, and also that judicious sampling of smaller column indices can be significantly more costly than, say, uniform random sampling.  Consequently, we will seek a trade-off for which both $|J|$ and $\beta$ are small compared to $n$, while simultaneously maintaining low complexity in the sampling scheme that produces the column indices $J$.

Let us now better illustrate the character of the bounds given in Theorem~\ref{THM:BetaBound} by some specific examples. \RV{We first consider deterministic sampling methods, and begin with an existential bound based on maximum volume sampling.}  The \textit{volume} of a matrix $M\in\R^{s\times t}$ is $\prod_{i=1}^{\min\{s,t\}}\sigma_i(M)$. 

\begin{proposition}{\cite[Lemma 1]{Osinsky2018}}
\label{COR:MaxVolBetaBound} Invoke the notations and assumptions of Theorem~\ref{THM:BetaBound}, and assume that $|J|=\ell$. If $J$ is chosen so that $\BV_{\BL}(J,:)$ is the maximal volume submatrix of $\BV_{\BL}$ of size $\ell\times r$, then
\[\beta \leq \sqrt{\frac{\ell}{n} + \frac{r\ell(n-\ell)}{n(\ell-r+1)} }.\]
\end{proposition}

Note that if $\ell=r$, then $\beta = \cO(r)$, whereas if $\ell=n$, then $\beta = 1$. Moreover, the function on the right-hand side of the estimate in Proposition \ref{COR:MaxVolBetaBound} is decreasing as $\ell$ increases for fixed $r$ and $n$.  Thus maximal volume sampling achieves a good bound for $\beta$.  However, these bounds come at substantial cost: finding maximal volume submatrices is NP--hard \cite{bartholdi1982good}.  There are approximation algorithms available \cite{goreinov2010find,mikhalev2018rectangular,osinsky2018rectangular}, but analysis of how close their output is to the optimal solution is unknown. 

Regarding deterministic column sampling methods, Avron and Boutsidis prove upper bounds on $\beta$ for a greedy column selection algorithm \cite{AB2013}.  We will discuss this algorithm in more detail later, but for now, we mention the following.
\begin{theorem}[{\cite[Theorem 3.1]{AB2013}}]
Let $\BL\in\mathbb{R}^{m\times n}$ with $\rank(\BL)=r$. Suppose $\BL$ satisfies \eqref{EQN:A1}.  Then Algorithm~\ref{ALG:Greedy} (\cite[Algorithm 1]{AB2013}) applied to $\BV_{\BL}$ yields column indices $J\subseteq[n]$ with $|J|=r$ for which ${\BC}=\BL(:,J)$ satisfies $\rank({\BC})=\rank({\BL})$, and $\beta \leq r.$
\end{theorem}
The greedy algorithm of \cite{AB2013} applied to $\BV_{\BL}$ requires $\cO(nr^2+nr(n-r))$ operations to select exactly $r$ columns and yield the bound $\beta\leq r$.  Thus, this method is more feasible than maximal volume sampling.  \RV{Both of the above procedures deterministically select exactly $r$ columns of $\BL$ and yield a good bound for $\beta$, but at nontrivial computational cost. We now turn our gaze to random sampling methods, which we will see can yield good bounds on $\beta$ with high probability at trivial computational cost.}

By far the easiest and cheapest sampling method is to choose $J$ from $[n]$ via uniform sampling.  While uniform sampling may not be successful for arbitrary matrices, it is known to be typically successful for incoherent matrices.  The following theorem derived from estimates of pseudoinverses of orthogonal matrices due to Tropp \cite{tropp2011improved} provides a bound for $\beta$ under this paradigm. 

\begin{theorem}\label{COR:BetaDeltaBound}
Suppose that $\BL\in\R^{m\times n}$ has rank $r$ and satisfies condition \eqref{EQN:A1}, and $\BV_{\BL}\in\R^{n\times r}$ are its first $r$ right singular vectors. Suppose that $J\subseteq[n]$ is chosen by sampling uniformly without replacement \RV{to yield $\BC=\BL(:,J)$} and that $|J|\geq\gamma\mu_2(\BL)r$ for some $\gamma>0$.  Then the quantity $\beta=\sqrt{|J|/n}\left\|(\BV_{\BL}(J,:))^\dagger\right\|_2$ satisfies
\begin{enumerate}[leftmargin=1cm]
    \item $\beta\leq\frac{1}{\sqrt{1-\delta}},$
    \item $\displaystyle\max_i\left\|\BV_{\BC}^T\bm{e}_i\right\|_2\leq  \beta\kappa(\BL)\sqrt{\frac{\mu_2(\BL)r}{|J|}},$ and 
    \item $\displaystyle\kappa(\BC)\leq \beta\sqrt{\mu_2(\BL) r}\kappa(\BL)$
\end{enumerate}
with probability of at least
$1-\frac{r}{e^{\gamma(\delta+(1-\delta)\log(1-\delta))}}$. In particular, $\mu_1(\BC)\leq\mu_1(\BL)$ and $\mu_2(\BC)\leq\beta^2\kappa(\BL)^2\mu_2(\BL)$ with the given success probability.
\end{theorem}

Let us provide a concrete choice of parameters for illustration, though we stress that the estimate in Theorem~\ref{COR:BetaDeltaBound} is quite flexible.   
\begin{corollary}\label{COR:UniformIncoherence}
With the notations and assumptions of Theorem~\ref{COR:BetaDeltaBound}, put $\delta=0.99$.  Then sampling 
$|J|\geq 1.06\mu_2(L) r\log(rn)$
columns from $[n]$ uniformly without replacement yields
\begin{enumerate}[leftmargin=1cm]
    \item $\beta\leq10,$
    \item 
    $\displaystyle\max_i\left\|V_{C}^Te_i\right\|_2\leq 10\kappa(L)\sqrt{\frac{\mu_2(L)r}{|J|}}  \leq 10\kappa(\BL)(\log(rn))^{-\frac12},$ and
    \item $\displaystyle\kappa(C)\leq 10\sqrt{\mu_2(L) r}\kappa(L)$
\end{enumerate}
with probability at least $1-\frac{1}{n}$.
In particular, $\mu_1(C)\leq\mu_1(L)$ and $\mu_2(C)\leq100\kappa(L)^2\mu_2(L)$ with the given success probability.
\end{corollary}

A brief remark here is needed on sampling order in Corollary~\ref{COR:UniformIncoherence} -- the $r\log(rn)$ term is mild, but necessary to achieve success probability which decays with the problem size, $n$; one can derive a similar result in which the success probability decays polynomially with the rank, $r$, while requiring only $|J|\gtrsim r\log(r)$; such a bound could be useful in problems in which the rank is not minuscule.  Later, this sampling order may be compared with those of CUR approximations, and will be seen to be favorable.

For ease of exposition above, we have focused solely on the case of column selection; however, we note that analogous results are easily obtained for row selection by taking transpose in the formulae above. For clarity, since we will subsequently be interested in selecting both column and row submatrices to form CUR decompositions, we state the combined bound on incoherence here.

\begin{corollary}\label{COR:CRUniformIncoherence}
With the notations and assumptions of Corollary~\ref{COR:UniformIncoherence}, suppose also that $I\subseteq[m]$ with $|I|\geq1.06\mu_1(\BL)r\log(rm)$
is sampled uniformly without replacement, and that $\BR=\BL(I,:)$.  If $\beta':=\sqrt{|I|/m}\left\|(\BW_{\BL}(I,:))^\dagger\right\|_2$, then with probability at least $(1-\frac{1}{n})(1-\frac{1}{m})$,
the conclusion of Corollary~\ref{COR:UniformIncoherence} holds, and so do the following:
\begin{enumerate}[leftmargin=1cm]
    \item $\beta'\leq 10$\\
    \item $\max_i\left\|\BV_{\BR}^T\bm{e}_i\right\|_2\leq \sqrt{\frac{\mu_2(\BL)r}{n}}$\\
    \item $\max_i\left\|\BW_{\BR}^T\bm{e}_i\right\|_2\leq 10\kappa(\BL)\sqrt{\frac{\mu_1(\BL)r}{|I|}}$\\
    \item $\kappa(\BR)\leq 10\sqrt{\mu_1(\BL)r}\kappa(\BL)$.
\end{enumerate}
In particular, $\mu_1(\BR)\leq 100\kappa(\BL)^2\mu_1(\BL)$, and $\mu_2(\BR)\leq \mu_2(\BL)$.
\end{corollary}

\begin{remark}
Note that if one used more sophisticated random sampling techniques, such as leverage score sampling \cite{DMPNAS} or volume sampling \cite{derezinski2018reverse}, one could potentially obtain good bounds on $\beta$ with larger success probability than in the results stated here. After all, these distributions take into account more information about the structure of the column and row space of the matrix. However, using these comes at a nontrivial computational cost for computation of the probability distributions ($\cO(mnk)$ for leverage scores and approximately $\cO(\max\{m^2,n^2\})$ for volume sampling).  Given that a constant bound for $\beta$ can be obtained with relatively large probability via uniform sampling, which requires no computation, this method should be preferred for incoherent matrices.  

There are faster methods than those mentioned above that approximate leverage score distributions, e.g., \cite{drineas2012fast}, but to be effective in the RPCA model, one needs to approximate leverage scores of $\BL$ from observations of $\BL+\BS$.  Leverage scores are unstable under addition of sparse, arbitrary magnitude outliers, so without imposing additional restrictions on $\BS$, such approximations are likely unsuitable for the RPCA problem.
\end{remark}

\subsection{Robust CUR Decomposition}\label{SEC:Main-RPCA}

Now we are ready to state our main algorithm.  The idea of Algorithm~\ref{ALG:Uniform} is simple, yet effective due to the computational ease of uniform sampling and the known success of CUR decompositions.  The algorithm may be considered a proto-type in that it has an RPCA algorithm as a subrouting in Lines~\ref{AlgLine:RPCA1} and \ref{AlgLine:RPCA2}.  We emphasize that any RPCA algorithm can be used in these lines, and we will quantify the accuracy and computational complexity of our algorithm based on those of the RPCA subroutine.  

\begin{algorithm}[h!]
 \caption{Uniform sampling Robust CUR}\label{ALG:Uniform}
\begin{algorithmic}[1]
\STATE\textbf{Input: }{$\BD=\BL+\BS\in\mathbb{R}^{m\times n}$ with $\rank(\BL)=r$; $\mathrm{RPCA}$: the chosen RPCA solver.
}

\STATE\textbf{Output: }{$\widehat{\BL}$: approximation of $\BL$.}

\STATE \textbf{Initialization: }{ Draw sampling indices $I\subseteq[m]$, $J\subseteq[n]$ uniformly.}
 
\STATE { $\widetilde{\BC} = \BD(:,J)$, $\widetilde{\BR} = \BD(I,:)$ }

\STATE{$[\widehat{\BC},\widehat{\BS}_{\BC}] = \mathrm{RPCA}(\widetilde{\BC},r) $}\label{AlgLine:RPCA1}

\STATE{$[\widehat{\BR},\widehat{\BS}_{\BR}] = \mathrm{RPCA}(\widetilde{\BR},r) $}\label{AlgLine:RPCA2}

\STATE{$\widehat{\BL} = \widehat{\BC}(\widehat{\BC}(I,:))^\dagger \widehat{\BR} 
$}
\end{algorithmic}
\end{algorithm}

\begin{remark}
We note that Algorithm~\ref{ALG:Uniform} and all of the theoretical results related to it here hold for both sampling column and row indices uniformly \textit{with} or \textit{without} replacement. For simplicity, we state all results here for sampling \textit{with} replacement.  The only minor change to any results here for sampling \textit{without} replacement is that some of the constants in the big-$\cO$ notations and success probabilities will change, but typically the change is small.
\end{remark}

Now let us state a sample result illustrating the type of guarantee Algorithm~\ref{ALG:Uniform} gives.  To do so, we assume that the $\mathrm{AltProj}$ RPCA algorithm of \cite{netrapalli2014non} is used in Lines~\ref{AlgLine:RPCA1} and \ref{AlgLine:RPCA2}.

\begin{theorem}\label{THM:MainErrorEstimate}
Let $\BL\in\R^{m \times n}$ and $\BS\in\R^{m \times n}$ satisfy Assumptions~\eqref{EQN:A1} and \eqref{EQN:A2}, respectively, with $\rank(\BL)=r$.
Let $\eps,\delta\in(0,1)$, and $\alpha=\cO\left(\frac{1-\delta}{\kappa(\BL)^2(\mu_1(\BL)\vee\mu_2(\BL))}\right)$. Suppose that 
\[|I|\geq  c_1\max\left\{\mu_1(\BL)r\log(m),\log(m)/\alpha \right\}, \qquad |J|\geq  c_2\max\left\{\mu_2(\BL)r\log(n),\log(n)/\alpha \right\},\]
rows and columns, respectively, are chosen uniformly without replacement, and that $\mathrm{AltProj}$ \cite{netrapalli2014non} is used as the $\mathrm{RPCA}$ subroutine in Lines~\ref{AlgLine:RPCA1} and \ref{AlgLine:RPCA2} of Algorithm~\ref{ALG:Uniform}. Then by running
$\cO\left(r\log\left(\sqrt{\frac{mn}{|I||J|}}\frac{\kappa(\BL)}{(1-\delta)\eps}\right)\right)$
iterations of $\mathrm{AltProj}$, the output of Algorithm~\ref{ALG:Uniform} satisfies
\[\frac{\left\|\BL-\widehat{\BL}\right\|_2}{\left\|\BL\right\|_2}\leq\eps \kappa(\BL)^{-1} \]
with probability at least
$1-\frac{r}{n^{c_2(\delta+(1-\delta)\log(1-\delta))}}-\frac{r}{m^{c_1(\delta-(1-\delta)\log(1-\delta))}}-\frac{1}{n}-\frac{1}{m}$.

Moreover, the total complexity of Algorithm~\ref{ALG:Uniform} with $\mathrm{AltProj}$ as a subroutine is 
\[\cO\left(\left(mr^3\log(n)+nr^3\log(m)\right)\log\left(\sqrt{\frac{mn}{|I||J|}}\frac{\kappa(\BL)}{(1-\delta)\eps}\right)\right).\]
\end{theorem}

\subsection{Hybrid Algorithms for Robust CUR Approximation}

The output of Algorithm~\ref{ALG:Uniform} is an approximation of the form $\BL\approx \widehat\BC\widehat\BU^\dagger \widehat\BR$, where $\widehat\BC$ and $\widehat\BR$ are cleaned versions of columns and rows, respectively, of $\BD$.  Put a different way, $\widehat\BC$ approximates the column space of $\BL$, and $\widehat\BR$ approximates the row space of $\BL$.  However, these consist of $\cO(r\log(rn))$ and $\cO(r\log(rm))$ vectors, respectively, while the dimensions of the column and row space of $\BL$ are both $r$; that is, we have redundant approximate bases for these spaces.  We are interested in determining from these the most relevant vectors that represent the largest subspaces of both the column and row space of $\BL$.  In particular, from the randomly sampled indices $I,J$, we would prefer to choose \textit{exactly} $r$ elements of each.  This problem in general is called the \textit{Column Subset Selection Problem}: given $\widehat\BC\in\R^{m\times \ell}$, choose $I_1\subseteq[\ell]$ (and set $\widehat{\BC}_1$) with $|I_1|=r$ which minimizes $\|\widehat{\BC}-\widehat{\BC}_1\widehat{\BC}_1^\dagger\widehat{\BC}\|_2$ over all possible column submatrices of $\widehat{\BC}_1$ consisting of exactly $r$ columns.  This problem is NP--complete \cite{Shitov} \RV{(see also, \cite{Civril} for other complexity results)}, but there are both randomized and deterministic approximation algorithms; a small subsampling of prior art on this problem is \cite{AB2013,AltschulerGreedyCSSP,BoutsidisNearOptimalCSSP,Deshpande2010,tropp2009column}. 

We consider now what happens when we apply a deterministic column \RV{subset} selection algorithm to the matrices $\widehat{\BC}$ and $\widehat{\BR}$ given by the output of Lines~\ref{AlgLine:RPCA1} and \ref{AlgLine:RPCA2} of Algorithm~\ref{ALG:Uniform} to choose exactly $r$ columns and rows thereof.  We state this algorithm generally, but specialize our analysis to the greedy algorithm of \cite{AB2013} (described fully here in the Appendix as Algorithm~\ref{ALG:Greedy}) to state a concrete theorem.  For clarity, we state the general algorithm here as Algorithm~\ref{ALG:Uniform+Deterministic}.  Again, RPCA is any such algorithm the user specifies, and $\mathrm{DeterministicCS}(\bm{A},r)$ is any algorithm that deterministically selects exactly $r$ columns of $\bm{A}$ ($I_1$ and $J_1$ in Algorithm~\ref{ALG:Uniform+Deterministic} are the indices chosen by the deterministic column selection algorithm).

\begin{algorithm}[h!]
 \caption{Uniform + Deterministic Robust CUR}\label{ALG:Uniform+Deterministic}
\begin{algorithmic}[1]
\STATE\textbf{Input: }{$\BD=\BL+\BS\in\mathbb{R}^{m\times n}$ with $\rank(\BL)=r$; $\mathrm{RPCA}$: the chosen RPCA solver; $\mathrm{DeterministicCS}$: the chosen deterministic row/column selection method. 
}

\STATE\textbf{Output: }{$\widehat{\BL}$: approximation of $\BL$.}

\STATE \textbf{Initialization: }{ Draw sampling indices $I\subseteq[m]$, $J\subseteq[n]$ uniformly.}
 
\STATE { $\widetilde{\BC} = \BD(:,J)$, $\widetilde{\BR} = \BD(I,:)$ }

\STATE{$[\widehat{\BC},\widehat{\BS}_{\BC}] = \mathrm{RPCA}(\widetilde{\BC},r) $}

\STATE{$[\widehat{\BR},\widehat{\BS}_{\BR}] = \mathrm{RPCA}(\widetilde{\BR},r) $}

\STATE{$[\widehat{\BC}_1,\widehat{\BS}_{\BC_1},J_1] = \mathrm{DeterministicCS}(\widehat{\BC},r)$}

\STATE{$[\widehat{\BR}_1,\widehat{\BS}_{\BR_1},I_1] = \mathrm{DeterministicCS}(\widehat{\BR}^T,r)$}\label{ALGLine:Det1}

\STATE{$\widehat{\BL} = \widehat{\BC}_1(\widehat{\BC}_1(I_1,:))^\dagger \widehat{\BR}_1 
$}\label{ALGLine:Det2}
\end{algorithmic}
\end{algorithm}

The following theorem describes the accuracy of the output of Algorithm~\ref{ALG:Uniform+Deterministic}.  

\begin{theorem}\label{THM:Hybrid}
Take the notations and assumptions of Theorem~\ref{THM:MainErrorEstimate}. 
Let $\widehat{\BC}$ and $\widehat{\BR}$ be the intermediate outputs of Algorithm~\ref{ALG:Uniform+Deterministic} in Lines~\ref{AlgLine:RPCA1} and \ref{AlgLine:RPCA2} with $\mathrm{AltProj}$ being used as the RPCA subroutine, where the latter runs $\cO\left(r\log\left(\frac{ 229\kappa(\BL)^5r\sqrt{mn\mu_1(\BL)\mu_2(\BL)}}{\varepsilon(1-\delta)^2\sigma_{\max}(\BL)} \right)\right)$ iterations.
Let $\widehat{\BC}_1=\widehat{\BC}(:,{J}_1)$, $\widehat{\BU}_1=\widehat{\BC}(I_1,:)$, and $\widehat{\BR}_1=\widehat{\BR}({I}_1,:)$ where ${I}_1$ and ${J}_1$ are obtained by applying Algorithm~\ref{ALG:Greedy} on $\widehat{\BR}$ and $\widehat{\BC}$ with $|{I}_1|=|{J}_1|=r$ in Lines~\ref{ALGLine:Det1} and \ref{ALGLine:Det2}. Then with probability at least $1-\frac{r}{n^{c_2(\delta+(1-\delta)\log(1-\delta))}}-\frac{r}{m^{c_1(\delta-(1-\delta)\log(1-\delta))}}-\frac{1}{n}-\frac{1}{m}$,
\[
    \left\|\BL-\widehat{\BC}_1\widehat{\BU}_1^\dagger\widehat{\BR}_1\right\|_2\leq \frac{1}{3}\varepsilon\sigma_{\min}(\BL).
\]
\end{theorem}

\begin{remark}
We note for the reader two things here. First, the proofs of the approximation theorems here are valid for the Frobenius norm as well as the smaller spectral norm; for presentation purposes we choose the latter.  Second, the bounds in Theorems \ref{THM:MainErrorEstimate} and \ref{THM:Hybrid} are the same by dividing by $\|\BL\|_2$; e.g., the bound in Theorem \ref{THM:Hybrid} is the same as the relative error bound
\[\frac{\|\BL-\widehat{\BC}_1\widehat{\BU}_1^\dagger\widehat{\BR}\|_2}{\|\BL\|_2}\leq \frac13\eps\kappa(\BL)^{-1}.\]
\end{remark}

\subsection{Deterministic Column Selection}

In Appendices~\ref{APP:Det1} and \ref{APP:Det2}, we study deterministic column sampling methods, in particular via the greedy algorithm of Avron and Boutsidis \cite{AB2013}.  We are able to prove that, with proper preconditioning, sampling columns and rows via the greedy algorithm on the singular vectors of $\BD$ produces submatrices of $\BL$ with similar incoherence (see Remark \ref{REM:GreedyBeta}). In particular, we find that the method described in Appendix~\ref{APP:Det1} produces deterministically a submatrix $\BC$ for which $\beta\leq 4\kappa(\BL)\sqrt{r}$.  However, this procedure does not guarantee $\cO(\alpha)$-sparsity of the corresponding submatrices of $\BS$. This we leave as an open problem for future work.

\section{Passing Properties to Submatrices}\label{SEC:Submatrices}

In this section, we study how three properties of a matrix pass to column submatrices thereof: incoherence, condition number, and sparsity.  These results are of independent interest, but are also necessary to study the use of Robust PCA algorithms on submatrices.

\subsection{Incoherence}

In this subsection, we study the relation of the incoherence of the matrix $\BL\in\mathbb{R}^{m\times n}$ to that of a column submatrix of it, say $\BC=\BL(:,J)\in\mathbb{R}^{m\times \ell}$. We find that column incoherence is unchanged, while row incoherence inflates by no more than a factor of $\beta\kappa(\BL)$. 

Let us begin by the observation that if $\BL$ has column incoherence as in the first inequality of \eqref{EQN:A1}, then any column submatrix $\BC$ has the same left incoherence.

\begin{lemma}\label{LEM:ColLeftCoherence}
Suppose ${\BL}\in\R^{m\times n}$ satisfies \eqref{EQN:A1}.  For any $J\subseteq[n]$ \RV{such that $\rank(\BC)=\rank(\BL)$ with $\BC={\BL}(:,J)$}, then $\BC$ satisfies
\[ \max_i\left\|{\BW}_{\BC}^T{\bm{e}}_i\right\|_2\leq\sqrt{\frac{\mu_1(\BL) r}{m}}. \]
\end{lemma}

\begin{proof}
Notice that $\BC={\BW}_{\BL}\BSigma_{\BL} ({\BV}_{\BL}(J,:))^T= {\BW}_{\BL}\widetilde{{\BW}}_{\BC}\BSigma_{\BC} {\BV}_{\BC}^T$,  where $\BSigma_{\BL}({\BV}_{\BL}(J,:))^T$ has the compact SVD decomposition $\widetilde{{\BW}}_{\BC}\BSigma_{\BC} {\BV}_{\BC}^T$ with  $\widetilde{{\BW}}_{\BC} \in\mathbb{R}^{r\times r}$ being an orthonormal matrix and ${\BV}_{\BC}^T\in\mathbb{R}^{r\times|J|}$. Therefore,   ${\BW}_{\BC}={\BW}_{\BL}\widetilde{{\BW}}_{\BC}$. Thus, we have
\[\max_{i}\left\|{{\BW}_{\BC}^T\bm{e}}_i\right\|_2=\max_{i}\left\|\widetilde{{\BW}}_{\BC}^T{\BW}_{\BL}^T\bm{e}_i\right\|_2=\max_{i}\left\|{\BW}_{\BL}^T\bm{e}_i\right\|_2\leq\sqrt{\frac{\mu_1(\BL) r}{m}}.
\]
\end{proof}

Lemma~\ref{LEM:ColLeftCoherence} implies that we need not be concerned with the column incoherence of any column submatrix $\BC$ of $\BL$ given that it always matches that of the matrix it is obtained from.  Now we need to understand the incoherence for ${\BV}_{\BC}$, which requires some more work.  First, let us make two preliminary observations.

\begin{lemma}\label{LEM:ColRightCoherence}
Let ${\BL}\in\R^{m\times n}$ with rank $r$ satisfy \eqref{EQN:A1}.  Then for any $J\subseteq[n]$, ${\BC}={\BL}(:,J)$ satisfies
\[ \max_i\left\|{\BV}_{\BC}^T\bm{e}_i\right\|_2\leq \kappa({\BL})\frac{\left\|{\BC}^\dagger\right\|_2}{\left\|{\BL}^\dagger\right\|_2}\sqrt{\frac{\mu_2(\BL) r}{n}}. \]
\end{lemma}
\begin{proof}
Notice that ${\BV}_{\BC}^T=\BSigma_{\BC}^{-1}{\BW}_{\BC}^T{\BC}$. Thus for any $i$,
\begin{align*}
\left\|{\BV}_{\BC}^T \bm{e}_i \right\|_2&=\left\|\BSigma_{\BC}^{-1}{\BW}_{\BC}^T{\BC}\bm{e}_i\right\|_2\\
& \leq\left\|\BSigma_{\BC}^{-1}\right\|_2\left\|{\BW}_{\BC}^T\right\|_2\left\|{\BC}\bm{e}_i\right\|_2\\
&\leq\frac{1}{\sigma_{\min}({\BC})}\left\|{\BC}\bm{e}_i\right\|_2.
\end{align*}
Since $\BC=\BL(:,J)=\BW_{\BL}\BSigma_{\BL}(\BV_{\BL}(J,:))^T$, so
\begin{align*}
\frac{1}{\sigma_{\min}(\BC)}\left\|\BC e_i\right\|_2 
&=\frac{1}{\sigma_{\min}({\BC})}\left\|\BW_{\BL}\BSigma_{\BL} (\BV(J,:))^T\bm{e}_i\right\|_2\\
&\leq \frac{\sigma_1(\BL)}{\sigma_{\min}({\BC})} \left\| (\BV_{\BL}(J,:))^T\bm{e}_i\right\|_2\\
&\leq \frac{\sigma_1(\BL)}{\sigma_{\min}({\BC})}\left\| \BV_{\BL}^T\bm{e}_i\right\|_2\\
& \RV{\leq}\kappa({\BL})\frac{\sigma_{\min}({\BL})}{\sigma_{\min}({\BC})}\sqrt{\frac{\mu_2(\BL)r}{n}}\\
&=\kappa({\BL})\frac{\left\|{\BC}^\dagger\right\|_2}{\left\|{\BL}^\dagger\right\|_2}\sqrt{\frac{\mu_2(\BL)r}{n}}.
\end{align*}
\end{proof}

Estimation of the quantity $\frac{\|{\BC}^\dagger\|_2}{\|{\BL}^\dagger\|_2}$ will generally depend heavily on how the column submatrix ${\BC}$ is selected. One way to estimate this is to consider the norm of the pseudoinverse of the corresponding column submatrix of ${\BV}_{\BL}$ as follows.

\begin{lemma}\label{LEM:CoverA}
Let ${\BL}\in\R^{m\times n}$ and ${\BC}={\BL}(:,J)$ for some $J\subseteq[n]$. Then
\[ \frac{\left\|{\BC}^\dagger\right\|_2}{\left\|{\BL}^\dagger\right\|_2} \leq \left\|\left({\BV}_{\BL}(J,:)\right)^\dagger\right\|_2. \]
\end{lemma}
\begin{proof}
Note that ${\BC} = {\BW}_{\BL}\BSigma_{\BL} \left({\BV}_{\BL}(J,:)\right)^T$.  Then one has
\[ {\BC}^\dagger = \left(\BSigma_{\BL}{\BV}_{\BL}(J,:)^T\right)^\dagger {\BW}_{\BL}^T = \left({\BV}_{\BL}(J,:)^T\right)^\dagger\BSigma_{\BL}^{-1}{\BW}_{\BL}^T, \]
where the first equality comes from the fact that ${\BW}_{\BL}$ has orthonormal columns and ${\BW}_{\BL}^\dagger={\BW}_{\BL}^T$, while the second equality follows from the fact that $\BSigma_{\BL}$ has full column rank and ${\BV}_{\BL}(J,:)^T$ has full row rank. With this in hand, we have
\[\left\|{\BC}^\dagger\right\|_2=\left\|\left(({\BV}_{\BL}(J,:))^T\right)^\dagger \BSigma_{{\BL}}^{-1}{\BW}_{\BL}^T\right\|_2\leq \left\|({\BV}_{\BL}(J,:)^\dagger)^T\right\|_2\left\|\BSigma_{{\BL}}^{-1}\right\|_2 =  \left\|{\BL}^\dagger\right\|_2\left\|({\BV}_{\BL}(J,:))^\dagger\right\|_2,
\] 
whence the result.
\end{proof}

Combining the conclusions of Lemmata~\ref{LEM:ColLeftCoherence}--\ref{LEM:CoverA},  we see that
\begin{equation}\label{EQN:CRightCoherenceGeneral}
    \max_i\left\|{\BV}_{\BC}^T\bm{e}_i\right\|_2 \leq  \kappa({\BL})\sqrt{\frac{|J|}{n}} \sqrt{\frac{\mu_2(\BL)r}{|J|}}\left\|({\BV}_{\BL}(J,:))^\dagger\right\|_2 = \beta\kappa(\BL)\RV{\sqrt{\frac{\mu_2(\BL)r}{|J|}} }.
\end{equation}

Our primary concern now is to attempt to bound the quantity $\beta=\sqrt{\frac{|J|}{n}}\left\|({\BV}_{\BL}(J,:))^\dagger\right\|_2$.  Unfortunately, there are relatively few general results for estimating pseudoinverses of submatrices of orthogonal matrices except in special instances.  We make this into the parameter $\beta$ of Theorem~\ref{THM:BetaBound}.

\subsection{Condition Number}

Many RPCA algorithms require an estimate on the condition number of the matrix that is input into the algorithm to provide good theoretical guarantees.  Here, we estimate $\kappa({\BC})$ in terms of $\kappa({\BL})$ in the case that $\BL$ has a certain incoherence level.  

\begin{lemma}\label{LEM:kCkAbound}
Let ${\BL}\in\R^{m\times n}$ have rank $r$ and satisfy assumption \eqref{EQN:A1}, and suppose $J\subseteq[n]$.  Then 
\begin{enumerate}[leftmargin=1cm]
    \item $\left\|{\BC}\right\|_2\leq\sqrt{\frac{\mu_2(\BL)r|J|}{n}}\left\|{\BL}\right\|_2$
    \item $\kappa({\BC})\leq  \sqrt{\mu_2(\BL) r}\kappa({\BL})\sqrt{\frac{|J|}{n}}\left\|{\BV}_{\BL}(J,:)^\dagger\right\|_2$.
\end{enumerate}
\end{lemma}
\begin{proof}
Since ${\BL}={\BW}_{\BL}\BSigma_{\BL} {\BV}_{\BL}^T$ and ${\BC}={\BL}(:,J)$, we have ${\BC}={\BW}_{\BL}\BSigma_{\BL} ({\BV}_{\BL}(J,:))^T$. For all $\bm{x}\in\mathbb{R}^{|J|}$ with $\left\|\bm{x}\right\|_2=1$,
\begin{eqnarray*}
\left\|{\BC}\bm{x}\right\|_2&=&\left\|{\BW}_{\BL}\BSigma_{\BL} ({\BV}_{\BL}(J,:))^T\bm{x}\right\|_2\\
&\leq&\left\|\BSigma_{\BL}\right\|_2\left\|({\BV}_{\BL}(J,:))^T\bm{x}\right\|_2\\
&=&\left\|{\BL}\right\|_2\left\|({\BV}_{\BL}(J,:))^T\sum_{i=1}^{|J|}\bm{x}_{i}\be_{i}\right\|_2\\
&\leq&\left\|{\BL}\right\|_2\sum_{i\in J}|\bm{x}_{i}|\left\|({\BV}_{\BL}(J,:))^T\be_{i}\right\|_2\\
&\leq&\left\|{\BL}\right\|_2\sum_{i\in J}|\bm{x}_i|\sqrt{\frac{\mu_2(\BL)r}{n}}\\
&\leq&\left\|{\BL}\right\|_2\sqrt{|J|}\|\bm{x}\|_2\sqrt{\frac{\mu_2(\BL)r}{n}}\\
&=&\sqrt{\frac{\mu_2(\BL)r|J|}{n}}\left\|\BL\right\|_2.
\end{eqnarray*}
To prove the second inequality, note that 
\[ \kappa({\BC}) = \left\|{\BC}\right\|_2\left\|{\BC}^\dagger\right\|_2\leq \sqrt{\frac{\mu_2(\BL)r|J|}{n}}\left\|{\BL}\right\|_2\left\|{\BV}_{\BL}(J,:)^\dagger\right\|_2\left\|{\BL}^\dagger\right\|_2,  \] which gives the required result upon rearranging terms.
\end{proof}

The results above combine to prove the main theorem (Theorem~\ref{THM:BetaBound}) regarding the quantity $\beta$.

\begin{proof}[Proof of Theorem~\ref{THM:BetaBound}]
The proof follows by combining Lemmata~\ref{LEM:ColLeftCoherence}--\ref{LEM:kCkAbound}, and \eqref{EQN:CRightCoherenceGeneral}.
\end{proof}

\subsection{Sparsity}

Our ability to apply a RPCA algorithm to a column submatrix of $\BD$ requires that the corresponding submatrix of $\BS$ remains a sparse matrix. The content of the following proposition is that uniform sampling of sufficiently many ($\cO(r\log n)$) columns yields a $2\alpha$--sparse submatrix if $\BS$ was $\alpha$--sparse.

\begin{proposition}\label{PROP:Sparsity}
Let $\BS$ be an $\alpha$-sparse matrix. We consider two cases of uniform sampling:
\begin{enumerate}[leftmargin=1cm]
\item (With replacement) Consider the submatrix of $\BS\in \R^{m\times n}$ formed by $|{I}|$ rows that were uniformly sampled \textit{with} replacement, namely $\BR$. Assume $|{I}|=cr\log(n)$ with $c \geq \frac{16}{3\alpha r}$, then $\BR\in\R^{|{I}| \times n}$ is $2\alpha$-sparse with probability at least $1-n^{-1}$. Similarly, $\BC\in\R^{m\times |J|}$, formed by $|J|=cr\log(m)$ columns of $\BS$ that were uniformly sampled \textit{with} replacement, is also $2\alpha$-sparse with probability at least $1-m^{-1}$. 
\item (Without replacement) Consider the submatrix of $\BS\in \R^{m\times n}$ formed by $|{I}|$ rows that were uniformly sampled \textit{without} replacement, namely $\BR$. Assume $|{I}|=cr\log(n)$ with $c \geq \frac{8}{\alpha r}$, then $\BR\in\R^{|{I}| \times n}$ is $2\alpha$-sparse with probability at least $1-2n^{-1}$. Similarly, $\BC\in\R^{m\times |J|}$, formed by $|J|=cr\log(m)$ columns of $\BS$ that were uniformly sampled \textit{without} replacement, is also $2\alpha$-sparse with probability at least $1-2m^{-1}$.
\end{enumerate}
\end{proposition}

\begin{proof}
\textbf{Part 1. } We first prove the version of uniform sampling \textit{with} replacement. By the definition of $\alpha$-sparsity, each row of $\BR$ has no more than $\alpha n$ non-zero entries. Moreover, assume the $j$--th column of $\BS$ has exact $\alpha n$ non-zero entries. Let $X_i$ indicate whether $\BR_{i,j}$ is non-zero. One can see that $X_i\sim\mathrm{Ber}(\alpha)$. Consequently, by Bernstein's inequality, we have
\begin{align*}
    \quad~\mathbb{P}\left( \sum_{i\in I} X_i >  2\alpha |I|  \right) 
    &=\mathbb{P} \left( \sum_{i\in I} X_i - \mathbb{E}(X_i) > \alpha|I|  \right) \cr
    &\leq \exp \left( - \frac{\alpha^2 |I|^2}{2|I|\alpha(1-\alpha)+\frac{2}{3}\alpha |I|}  \right) \cr
    &\leq \exp \left( -\frac{\alpha |I|}{2(1-\alpha)+\frac{2}{3}} \right) \cr
    &\leq \exp\left( -2 \log(n) \right)\cr
    &=n^{-2}
\end{align*}
for the $j$--th column of $\BR$, where the last inequality uses the assumption $\alpha |I|\geq \frac{16}{3} \log(n)$. Note that the probability inequality still holds if $\BS$ has less than $\alpha n$ non-zero entries. Furthermore, to have all $n$ columns being $2\alpha$--sparse, we get
\begin{align*}
    \mathbb{P}(\BR \textnormal{ is } 2\alpha\textnormal{-sparse})\geq (1-n^{-2})^n\geq 1-n^{-1}.
\end{align*}
Hence, $\BR$ is $2\alpha$--sparse with high probability; the proof that $\BC$ is $2\alpha$--sparse with high probability is the same and so is omitted.

\textbf{Part 2. } Now, we prove the version of uniform sampling \textit{without} replacement.  By  \cite[Theorem~1]{gross2010note}, the Bernstein inequality for uniform sampling \textit{without} replacement is:
\[
\mathbb{P}\left(\sum_{i=0}^m Y_i>t \right)\leq 2 \exp\left(-\frac{t^2}{4m \sigma^2}\right)
\]
where $Y_i$ is zero-mean random variables with uniform probability and $\sigma$ is the variance. 
Taking the same argument as in Part $1$ except sampling uniformly \textit{without} replacement, we have 
\begin{align*}
    \quad~\mathbb{P}\left( \sum_{i\in I} X_i >  2\alpha |I|  \right) 
    &=\mathbb{P} \left( \sum_{i\in I} X_i - \mathbb{E}(X_i) > \alpha|I|  \right) \cr
    &\leq 2\exp \left( - \frac{\alpha^2 |I|^2}{4|I|\alpha(1-\alpha)}  \right) \cr
    &\leq 2\exp \left( - \frac{\alpha |I|}{4(1-\alpha)}  \right) \cr
    &\leq 2\exp\left( -2 \log(n) \right)\cr
    &=2n^{-2}
\end{align*}
for the $j$--th column of $\BR$, where the last inequality uses the assumption $\alpha |I|\geq  8 \log(n)$. Therefore, 
\begin{align*}
    \mathbb{P}(\BR \textnormal{ is } 2\alpha\textnormal{-sparse})\geq (1-2n^{-2})^n\geq 1-2n^{-1}
\end{align*}
for $\BR$ generated by uniform sampling \textit{without} replacement. 
Similar result holds for $\BC$.
\end{proof}

\section{Simulations}\label{SEC:Simulations}
In this section, the effectiveness and computational efficiency of RCUR (i.e., Algorithm~\ref{ALG:Uniform})  is illustrated on two key imaging applications, video background-foreground separation and face modeling, and that can be summarized in the following three ways: 
\begin{enumerate}
\item We show that running time of RCUR on a video background subtraction task is typically an order of magnitude faster than the state-of-the-art RPCA algorithm. Qualitatively, we also show that the output produced by RCUR is similar to that of RPCA benchmarks.
\item We illustrate the interpretable nature of the dictionary vectors produced by RCUR.
\item We show that the running time of RCUR on face modeling is typically faster than that of RPCA and the output produced by RCUR is comparable to that of RPCA. 
\end{enumerate}  
Moreover, for fair comparison, we use AccAltProj\footnote{AccAltProj \cite{cai2019accelerated} is a fast state-of-the-art RPCA algorithm that accelerates the popular AltProj \cite{netrapalli2014non}.} as the RPCA algorithm in all the numerical experiments, including the RPCA subroutine that is used in Lines~5--6 of Algorithm~\ref{ALG:Uniform} and \ref{ALG:Uniform+Deterministic}.

\subsection{Application to Video Background-Foreground Separation} \label{sec:video bf separation}

Background-foreground separation is  an important tool to extract moving objects from a static background in a video \cite{BT2009}. In this section, we compare the performance of the Robust CUR decomposition with that of RPCA on three public benchmark videos: \textit{Shoppingmall}, \textit{Restaurant} and \textit{OSU} datasets\footnote{\url{http://perception.i2r.a-star.edu.sg/bk\_model/bk\_index.html}.}. The size of each frame of \textit{Shoppingmall} is $256 \times 320$, that
of\textit{ Restaurant} is $120 \times 160$ and that of OSU is $240\times 320$. The total number of frames are 1000 for \textit{Shoppingmall}, 3055 for \textit{Restaurant}, and $1506$ for \textit{OSU}. Each video can be represented as a matrix $\BD\in\mathbb{R}^{m\times n}$ by vectorizing and stacking the frames as columns of the matrix; thus the \textit{Shoppingmall} video produces a matrix in $\R^{81,920\times 1,000}$ while \textit{Restaurant}'s is in $\R^{19,200\times 3,055}$ and \textit{OSU}'s is in $\R^{76,800\times 1506}$.  Then $\BD=\BL+\BS$, where $\BL$ is the background, and $\BS$ contains the foreground elements treated as sparse outliers.

\subsubsection{Runtime and Quality}

We apply Algorithm~\ref{ALG:Uniform}  to $\BD$, the data matrix from  \textit{Shoppingmall}, \textit{Restaurant} and \textit{OSU}.  We take $r=2$ to be the underlying rank of the background for each video sequence.  Then according to Theorem~\ref{THM:MainErrorEstimate}, we randomly choose column $\widetilde{\BC}=\BD(:,J)$ and row $\widetilde{\BR}=\BD(I,:)$ submatrices of size
$m\times 15 r\log(n)$ and $25r\log(m)\times n$,
respectively. Note that $\mu_1(\BL)\neq\mu_2(\BL)$ in many real-world applications, and $\mu_1(\BL)>\mu_2(\BL)$ under this problem setup. Then we apply the selected RPCA algorithm (i.e., AccAltProj \cite{cai2019accelerated})  on $\widetilde{\BC}$ and $\widetilde{\BR}$ (Lines~\ref{AlgLine:RPCA1} and \ref{AlgLine:RPCA2}) and denote the corresponding output low-rank matrices by $\widehat{\BC}$ and $\widehat{\BR}$. We then approximate the static background of the video, which is a low-rank matrix $\BL$, as $\BL\approx\widehat{\BL}=\widehat{\BC}(\widehat{\BC}(I,:))^\dagger \widehat{\BR}$. The foreground $\BS$, which represents the moving objects, is then approximated by $\BD-\widehat{\BL}$. 

\begin{table}[h]
\caption{Video size and runtime. }\label{tab:rt_shoppingmall}
 \centering
 \begin{tabular}{ |c||c|c|c|c|c|} 
\hline
 ~             &frame & frame &   \multicolumn{2}{c|}{runtime (sec)}\cr
 \cline{4-5}

~             & size                           & number                             & RCUR& RPCA            \cr

 \hhline {|=||=|=|=|=|=|}

\textbf{Shoppingmall}  &$256\times 320$              & $1000$                     &   $7.69$ &  $44.30$   \cr
\hline
\textbf{Restaurant}    &$120\times 160$              & $3055$                  &   $3.48$   &  $31.63$  \cr
\hline
\textbf{OSU} &$240\times 320$ &$1506$&10.39&68.62\cr
\hline
\end{tabular}
\end{table}

Table~\ref{tab:rt_shoppingmall} shows the runtime for Algorithm~\ref{ALG:Uniform} as well as the runtime for applying RPCA directly on $\BD$.  In all cases, Algorithm~\ref{ALG:Uniform} runs substantially faster; in particular, on most problems, the runtime is a full order of magnitude faster than that of RPCA.

\begin{figure}[h!]
\vspace{-0.02in}
\centering
\subfloat{\includegraphics[width=.195\linewidth]{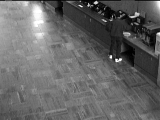}} \hfill
\subfloat{\includegraphics[width=.195\linewidth]{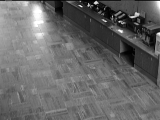}} \hfill
\subfloat{\includegraphics[width=.195\linewidth]{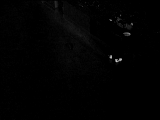}} \hfill
\subfloat{\includegraphics[width=.195\linewidth]{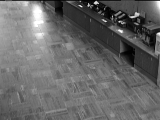}} \hfill
\subfloat{\includegraphics[width=.195\linewidth]{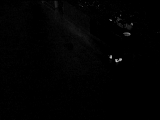}}
\vspace{-0.12in}

\subfloat{\includegraphics[width=.195\linewidth]{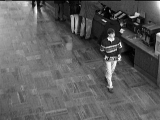}} \hfill
\subfloat{\includegraphics[width=.195\linewidth]{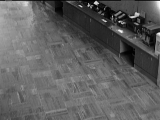}} \hfill
\subfloat{\includegraphics[width=.195\linewidth]{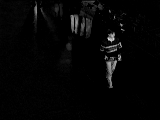}} \hfill
\subfloat{\includegraphics[width=.195\linewidth]{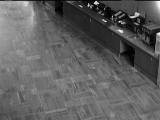}} \hfill
\subfloat{\includegraphics[width=.195\linewidth]{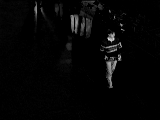}}
\vspace{-0.12in}
\subfloat{\includegraphics[width=.195\linewidth]{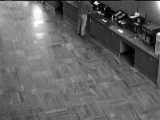}} \hfill
\subfloat{\includegraphics[width=.195\linewidth]{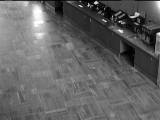}} \hfill
\subfloat{\includegraphics[width=.195\linewidth]{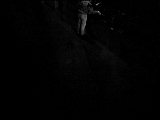}} \hfill
\subfloat{\includegraphics[width=.195\linewidth]{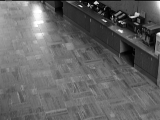}} \hfill
\subfloat{\includegraphics[width=.195\linewidth]{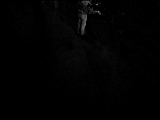}}
\caption{\textit{Restaurant}: The first column contains three randomly selected frames from the original video. The middle two columns are the separated background
and foreground outputs of RCUR, respectively. The right two columns are the separated background and foreground outputs of RPCA, respectively.} \label{FIG:video background subtraction-restaurant}
\vspace{-0.05in}
\end{figure}

\begin{figure}[h!]
\vspace{-0.02in}
\centering
\subfloat{\includegraphics[width=.195\linewidth]{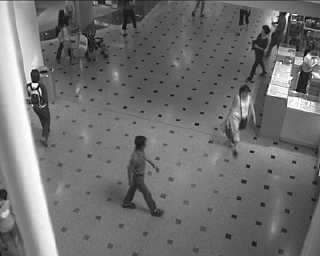}} \hfill
\subfloat{\includegraphics[width=.195\linewidth]{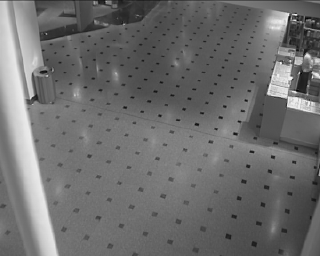}} \hfill
\subfloat{\includegraphics[width=.195\linewidth]{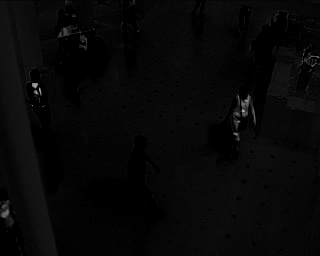}} \hfill
\subfloat{\includegraphics[width=.195\linewidth]{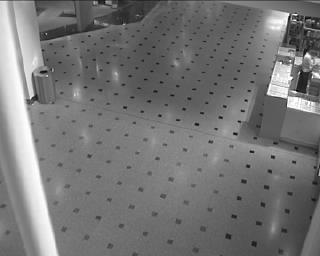}} \hfill
\subfloat{\includegraphics[width=.195\linewidth]{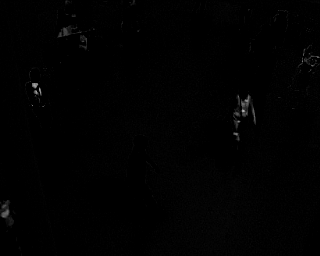}}
\vspace{-0.12in}

\subfloat{\includegraphics[width=.195\linewidth]{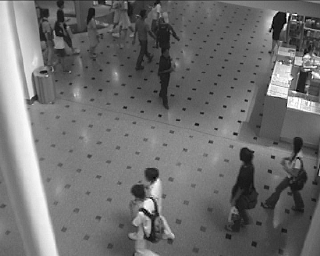}} \hfill
\subfloat{\includegraphics[width=.195\linewidth]{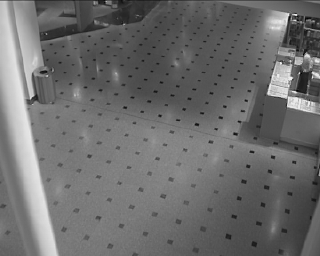}} \hfill
\subfloat{\includegraphics[width=.195\linewidth]{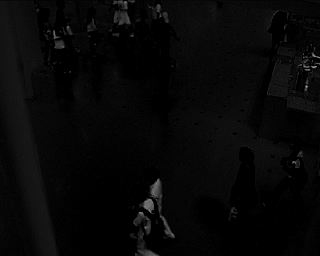}} \hfill
\subfloat{\includegraphics[width=.195\linewidth]{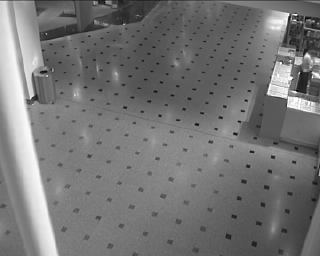}} \hfill
\subfloat{\includegraphics[width=.195\linewidth]{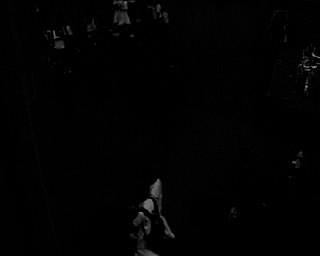}}
\vspace{-0.12in}
\subfloat{\includegraphics[width=.195\linewidth]{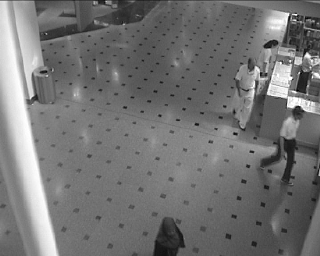}} \hfill
\subfloat{\includegraphics[width=.195\linewidth]{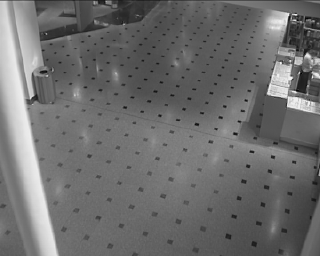}} \hfill
\subfloat{\includegraphics[width=.195\linewidth]{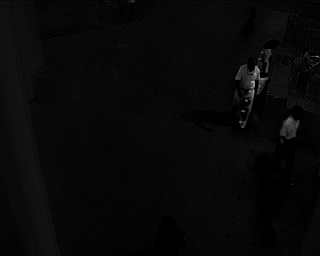}} \hfill
\subfloat{\includegraphics[width=.195\linewidth]{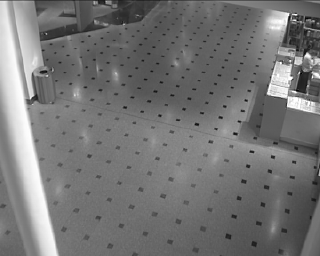}} \hfill
\subfloat{\includegraphics[width=.195\linewidth]{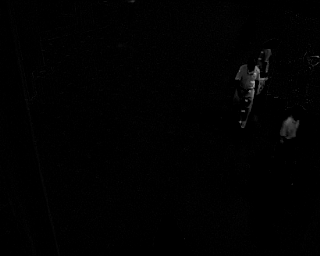}}

\caption{\textit{Shoppingmall}: The first column contains three randomly selected frames from the original video. The middle two columns are the separated background
and foreground outputs of RCUR, respectively. The right two columns are the separated background and foreground outputs of RPCA, respectively.} \label{FIG:video background subtraction-shoppingmall}
\vspace{-0.05in}
\end{figure}

\begin{figure}[h!]
\vspace{-0.02in}
\centering
\subfloat{\includegraphics[width=.195\linewidth]{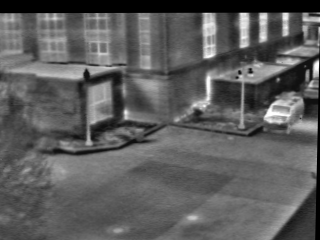}} \hfill
\subfloat{\includegraphics[width=.195\linewidth]{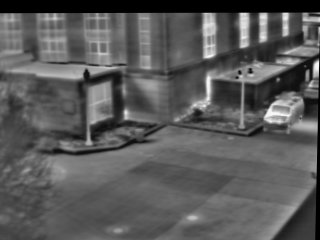}} \hfill
\subfloat{\includegraphics[width=.195\linewidth]{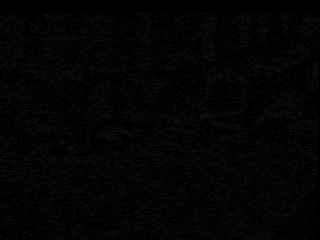}} \hfill
\subfloat{\includegraphics[width=.195\linewidth]{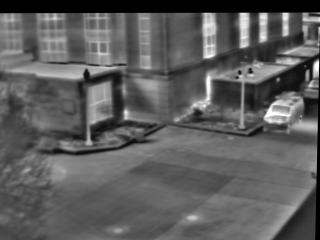}} \hfill
\subfloat{\includegraphics[width=.195\linewidth]{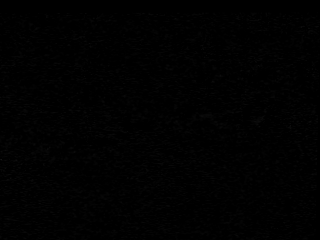}}
\vspace{-0.12in}

\subfloat{\includegraphics[width=.195\linewidth]{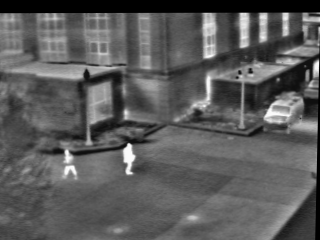}} \hfill
\subfloat{\includegraphics[width=.195\linewidth]{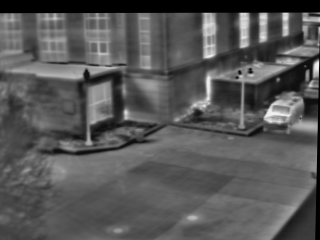}} \hfill
\subfloat{\includegraphics[width=.195\linewidth]{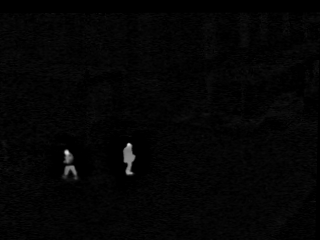}} \hfill
\subfloat{\includegraphics[width=.195\linewidth]{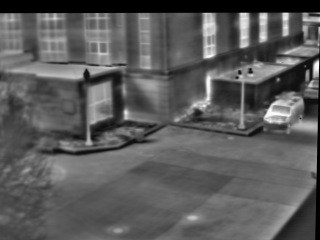}} \hfill
\subfloat{\includegraphics[width=.195\linewidth]{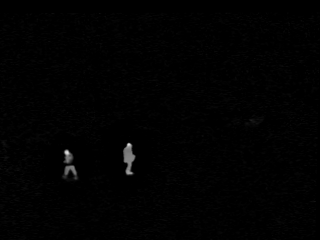}}
\vspace{-0.12in}
\subfloat{\includegraphics[width=.195\linewidth]{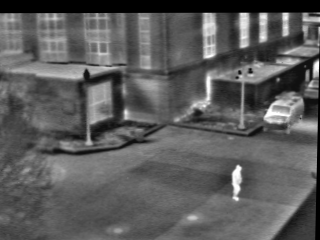}} \hfill
\subfloat{\includegraphics[width=.195\linewidth]{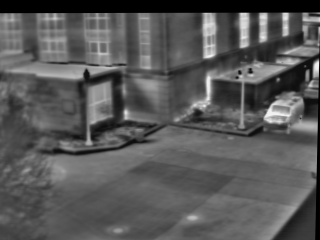}} \hfill
\subfloat{\includegraphics[width=.195\linewidth]{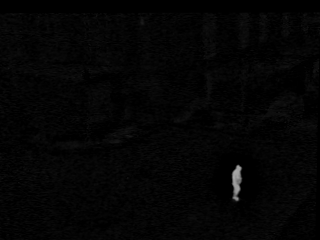}} \hfill
\subfloat{\includegraphics[width=.195\linewidth]{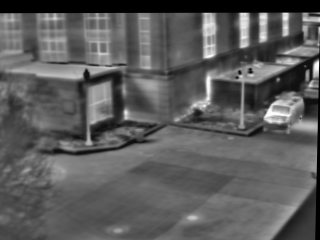}} \hfill
\subfloat{\includegraphics[width=.195\linewidth]{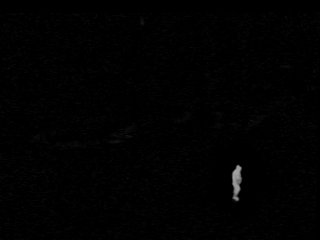}}

\caption{\textit{OSU}: The first column contains three randomly selected frames from the original video. The middle two columns are the separated background
and foreground outputs of RCUR, respectively. The right two columns are the separated background and foreground outputs of RPCA, respectively.} \label{FIG:video background subtraction-OSU}
\vspace{-0.05in}
\end{figure}

Running time is not the whole story, as the quality of the separation of foreground and background is of prime importance as well. Note that in the background-foreground separation tasks within the RPCA framework, there is no ground truth, and thus comparisons must be done qualitatively. Figures~\ref{FIG:video background subtraction-restaurant},  \ref{FIG:video background subtraction-shoppingmall}, and \ref{FIG:video background subtraction-OSU} illustrate the outputs of RCUR and RPCA on  \textit{Restaurant}, \textit{Shoppingmall}, and \textit{OSU} sequences, respectively.  Each figure presents three randomly selected frames from the full video sequence as well as the background and foreground images from each algorithm.  In both cases, we note that the reconstruction algorithms yield comparable results; thus, the significant speed advantage afforded by RCUR is desirable in this particular Robust PCA task.

\subsubsection{Interpretability -- Extracting Canonical Frames}
\begin{figure}[h!]
\vspace{-0.02in}
\centering
\subfloat{\includegraphics[width=.235\linewidth]{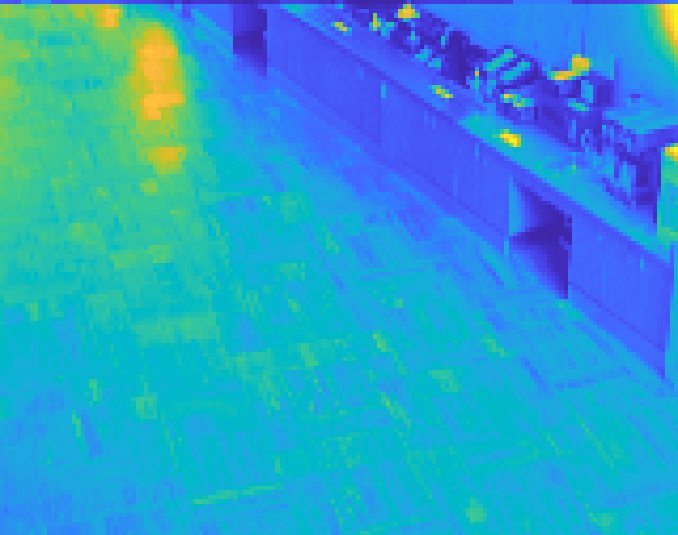}} \hfill
\subfloat{\includegraphics[width=.235\linewidth]{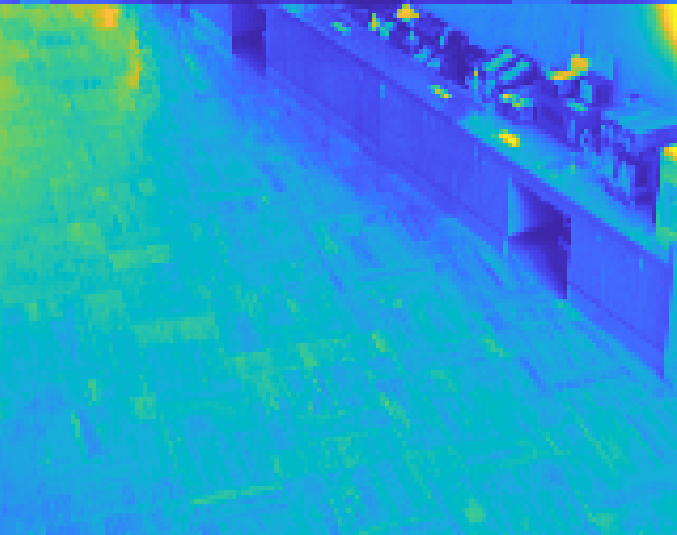}} \qquad
\subfloat{\includegraphics[width=.235\linewidth]{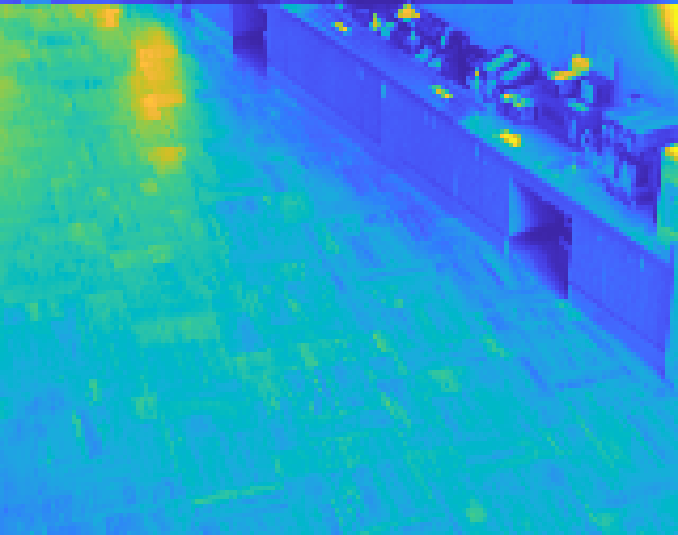}} \hfill
\subfloat{\includegraphics[width=.235\linewidth]{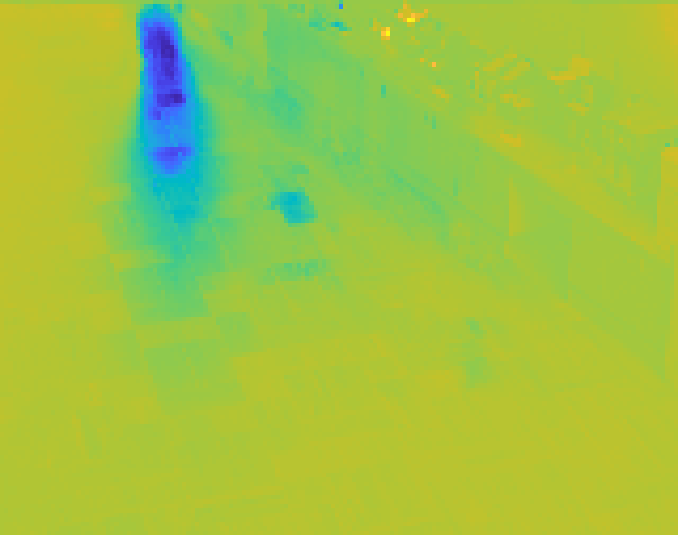}} 
\vspace{-0.12in}

\subfloat{\includegraphics[width=.235\linewidth]{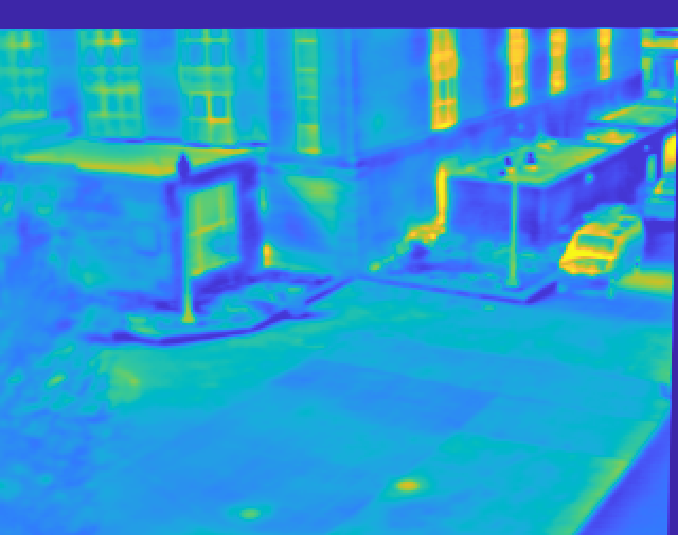}} \hfill
\subfloat{\includegraphics[width=.235\linewidth]{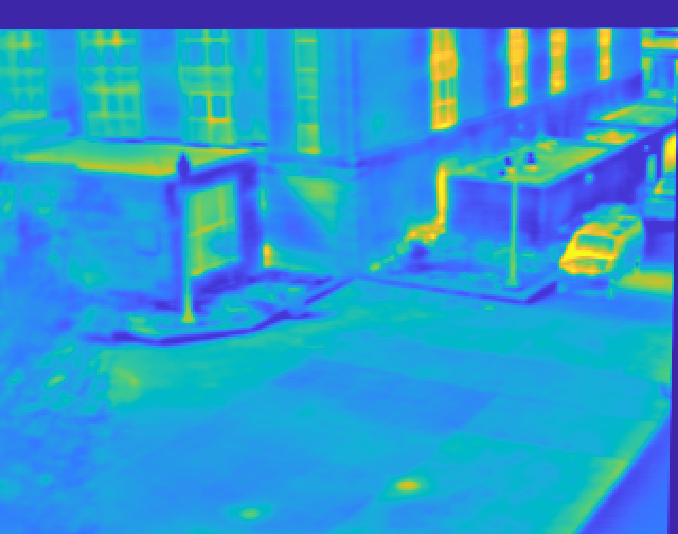}} \qquad
\subfloat{\includegraphics[width=.235\linewidth]{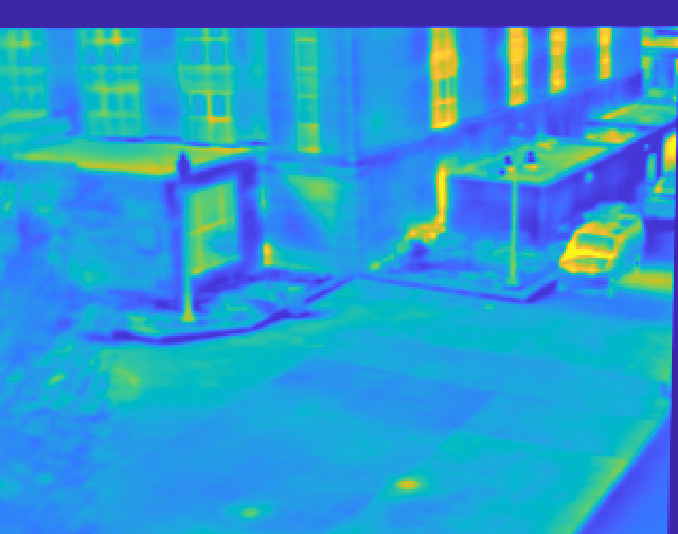}} \hfill
\subfloat{\includegraphics[width=.235\linewidth]{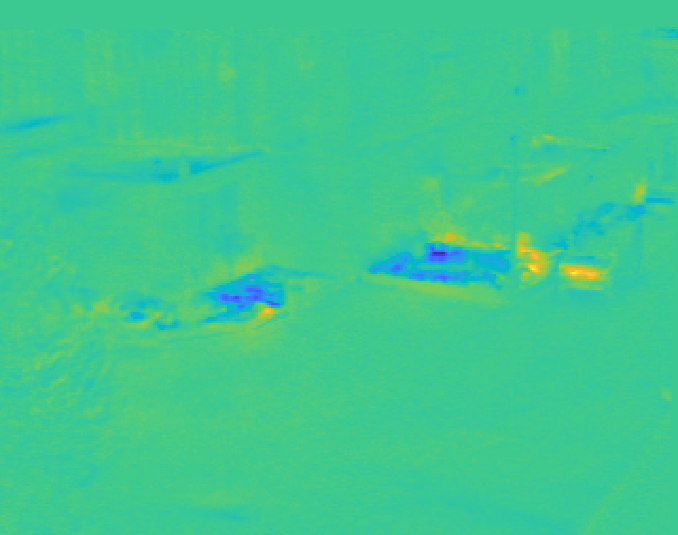}} 
\vspace{-0.12in}

\subfloat{\includegraphics[width=.235\linewidth]{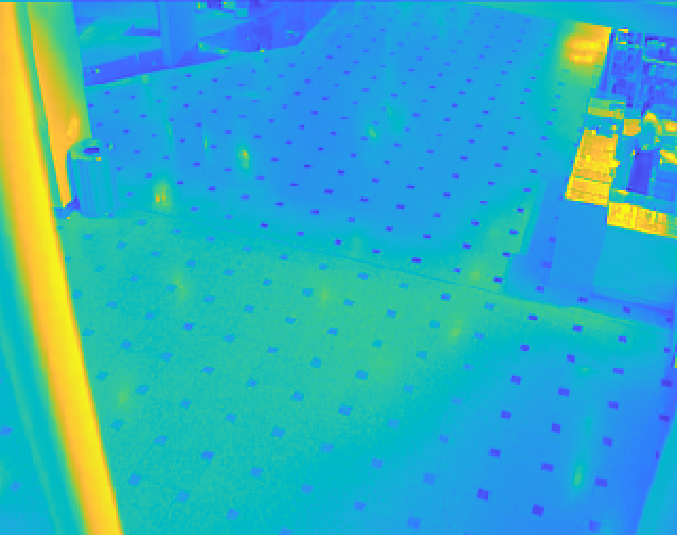}} \hfill
\subfloat{\includegraphics[width=.235\linewidth]{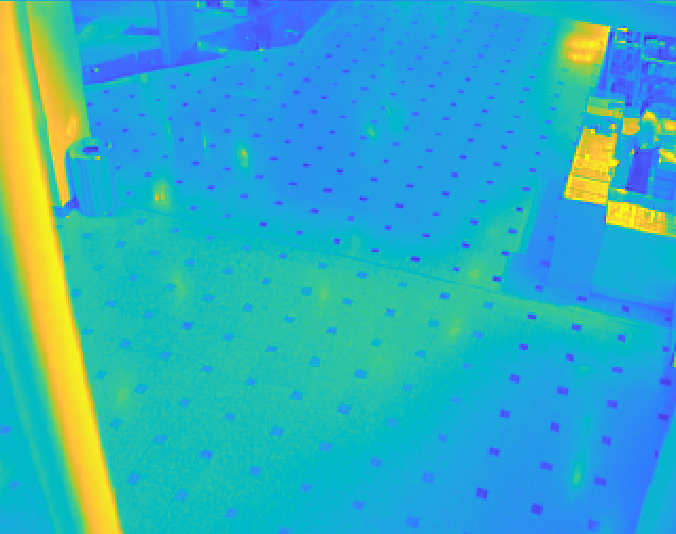}} \qquad
\subfloat{\includegraphics[width=.235\linewidth]{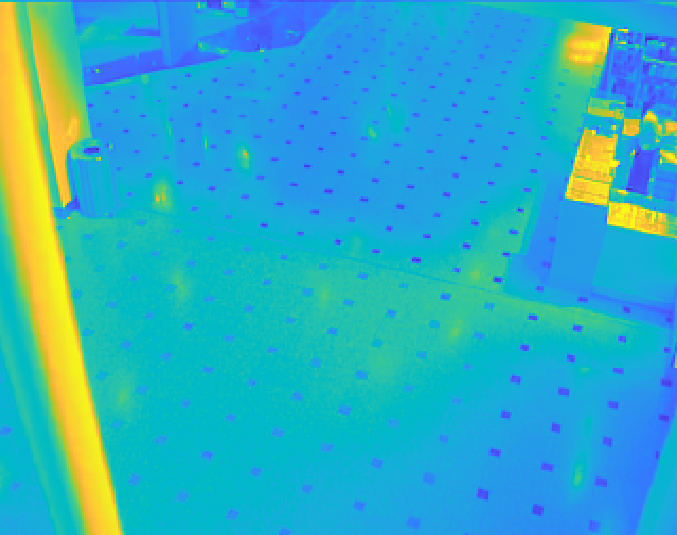}} \hfill
\subfloat{\includegraphics[width=.235\linewidth]{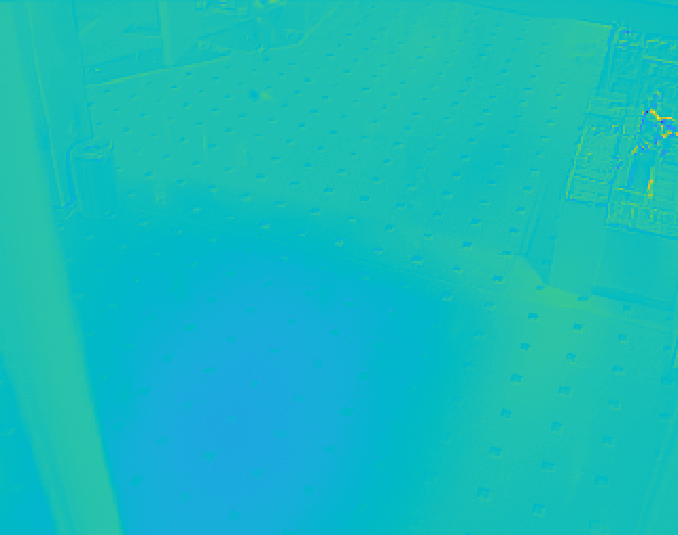}} 

 \caption{ Basis vectors for background of video sequences (Row 1:\textit{ Restaurant}, Row 2: \textit{OSU}, Row 3: \textit{Shoppingmall}).  In each row, the first two columns are the selected frames by using the deterministic CUR algorithm described above, and the last two columns are the first two orthogonal basis vectors of the SVD of the recovered low rank matrix, respectively.}\label{FIG:video background subtraction basis}
\vspace{-0.05in}
\end{figure}

CUR decompositions have long been touted as providing better interpretability of data representations than abstract bases such as those obtained by QR or SVD (see, e.g., \cite{HammHuang,DMPNAS,SorensenDEIMCUR}), the idea being that a subset of the data itself is a good dictionary for the whole.  This is also known as the self-expressivity property of data, and has been used to good effect in solving the subspace clustering problem \cite{AHKS,aldroubi2018similarity,SSC,haeffele2020critique,liu2012robust,vidal2011subspace}.  Additionally, as pointed out in \cite{SorensenDEIMCUR}, column selection can produce a dictionary which explains multiple variances which are not orthogonal to each other, in contrast to standard PCA.  

In the context of background-foreground separation, we explore this aspect of Robust CUR decompositions.  To do so, we use Algorithm~\ref{ALG:Uniform+Deterministic} to first uniformly sample columns (i.e., frames) of the video sequences according to the sampling requirement of Theorem~\ref{THM:MainErrorEstimate}, then following background-foreground separation via Robust PCA, we apply the deterministic sampling method of Algorithm~\ref{ALG:Greedy} to select exactly $r=2$ frames from within those chosen in the first stage.  Resulting from this procedure, we obtain two frames of the video background which are highly representative of the rest of the frames; in particular, they approximately minimize $\left\|\BL-\BC\BC^\dagger\BL\right\|_F$, the error when projecting the rest of the video frames onto the chosen two.  Figure~\ref{FIG:video background subtraction basis} shows the canonical background frames selected from this procedure compared with the first two basis vectors obtained by taking the SVD of the video matrix $\BD$.

Note that in the first row of Figure~\ref{FIG:video background subtraction basis}, the left two panes are the two background frames obtained by Algorithm~\ref{ALG:Uniform+Deterministic}, and represent the two canonical states of the background: one with light shining through the window onto the floor, and one with the light occluded.  In contrast, the second SVD basis vector  captures only the light portion of the background. In the other video sequences, the background is relatively static, and so our Robust CUR method produces two very similar canonical background frames.  In contrast to this, the second SVD basis vectors are uninformative for these sequences as one might expect. Consequently, the CUR basis vectors capture the essential features of the different states of the background, and are easily interpreted by the user.  In contrast to this, the SVD basis vectors may appear meaningless in terms of the actual background.  This difference stems from the fact that the column vectors are approximating actual frames, whereas the SVD is capturing orthogonal images corresponding to decreasing variance.

\subsection{Application to Face Modeling}
In this section, we consider the problem of robust face modeling. We use the Extended Yale Face Database B (abbr. \textit{ExtYaleB}) \cite{GeBeKr01} for a benchmark, which includes the face data of $28$ human objects under $64$ illumination conditions and each face image has size of $168\times 192$. We vectorize the face images and stack the faces of the person together to form a data matrix, so we have $28$ data matrices in $\mathbb{R}^{32,256\times 64}$. Then in terms of $\BD=\BL+\BS$, the underlying face models form the low-rank $\BL$ and the facial occlusions are sparse outlier $\BS$. 
We apply Algorithm~\ref{ALG:Uniform} to each of the face data matrix $\BD$ where $r=1$ is used to extract the face model. Again, since $\mu_2(\BL)>\mu_1(\BL)$ in this problem, we randomly choose column $\widetilde{\BC}=\BD(:,J)$ and row $\widetilde{\BR}=\BD(I,:)$ submatrices of size $m\times 10 r\log(n)$ and $25r\log(m)\times n$, respectively. The rest of the experimental setup is same as Section~\ref{sec:video bf separation}. 

In Table~\ref{tab:face_model}, we report the total runtime for Algorithm~\ref{ALG:Uniform} and RPCA on the face modeling task. Clearly, Algorithm~\ref{ALG:Uniform} runs faster than RPCA but the advantage is not as large as in the video background-foreground separation task. The reason is that the face data matrix is too rectangular (i.e., $n=64$ is too small), which gives less room for acceleration. Moreover, we present the visual face modeling results of a selected human object in Figure~\ref{FIG:face_modeling}, wherein we find both algorithms achieve the desired modeling quality. 
\begin{table}[h]
\caption{Face data size and runtime. }\label{tab:face_model}
 \centering
 \begin{tabular}{ |c||c|c|c|c|c|} 
\hline
 ~             &image & image number &  person & \multicolumn{2}{c|}{total runtime (sec)} \cr
 \cline{5-6}

~             & size                           & per person & number                            & RCUR& RPCA            \cr

 \hhline {|=||=|=|=|=|=|}

\textbf{ExtYaleB}  &$168\times 192$              & $64$  & $28$                   &   $20.16$ &  $31.38$     \cr
\hline
\end{tabular}
\end{table}

\begin{figure}[!h]
\vspace{-0.02in}
\centering
\includegraphics[width=\linewidth]{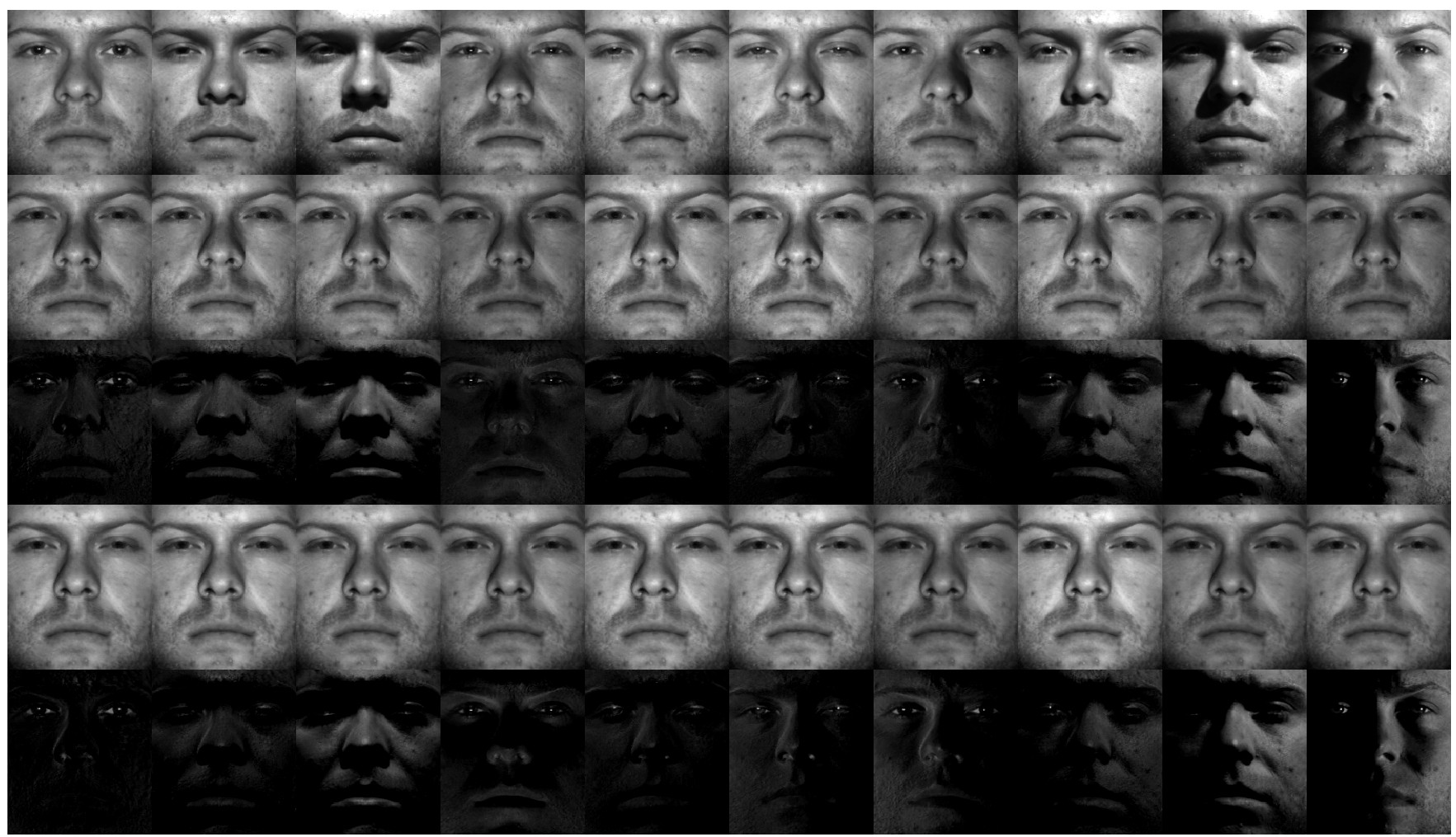}
\caption{Face modeling on \textit{ExtYaleB}: Visual comparison of the outputs by RCUR and RPCA for face modeling task. The first row contains the original face images. The second and third rows are the face models and the facial occlusions outputted by RCUR, respectively. The last two rows are the face models and the facial occlusions outputted by RPCA, respectively. }\label{FIG:face_modeling}
\vspace{-0.05in}
\end{figure}


\bibliographystyle{siamplain}
\bibliography{ref}

\newpage
\appendix

\section{Uniform Sampling and the Proof of Theorem~\ref{COR:BetaDeltaBound}}\label{APP:ProofBeta}

Here, we use a special case of the analysis of Tropp \cite{tropp2011improved} to estimate the quantity $\beta$ in Theorem~\ref{THM:BetaBound}.

\begin{theorem}[{\cite[Lemma 3.4]{tropp2011improved}}]\label{THM:UniformBeta}
Suppose that $\BL\in\R^{m\times n}$ has rank $r$ and satisfies condition \eqref{EQN:A1}, and ${\BV}_{\BL}\in\R^{n\times r}$ are its first $r$ right singular vectors.  Suppose that ${J}\subseteq[n]$ is chosen by sampling uniformly without replacement and that $|{J}|\geq \gamma\mu_2(\BL) r$ for some $\gamma>0$.  Then the quantity $\beta = \sqrt{|{J}|/n} \left\|({\BV}_{\BL}({J},:))^\dagger\right\|_2$ satisfies
\[ \beta \leq \frac{1}{\sqrt{1-\delta}}\quad \textnormal{with probability at least}\quad
1-r\left(\frac{e^{-\delta}}{(1-\delta)^{1-\delta}}\right)^\gamma,
\quad \textnormal{for all}\quad \delta\in[0,1). \]
\end{theorem}

\begin{proof}[Proof of Theorem~\ref{COR:BetaDeltaBound}]
Combine Theorems~\ref{THM:UniformBeta} and \ref{THM:BetaBound}.
\end{proof}

\section{Error Analysis and Proof of Theorem~\ref{THM:MainErrorEstimate}}\label{APP:ProofMain}

In this section we provide a combined error analysis of RPCA (in particular, AltProj \cite{netrapalli2014non}) and CUR approximations to obtain the main result concerning Algorithm~\ref{ALG:Uniform}, namely Theorem~\ref{THM:MainErrorEstimate}.

\subsection{RPCA}
Convergence has been widely studied for many RPCA algorithms \cite{netrapalli2014non}. Here, we list one for $\mathrm{AltProj}$ \cite{netrapalli2014non}, which is one of the classic non-convex RPCA algorithm. We emphasize that our results will stand with other RPCA convergence theories with slight variation.

\begin{theorem}[{\cite[Theorem~1]{netrapalli2014non}}]\label{thm: rpca}
Let $\BL$ and $\BS$ satisfy Assumption~\eqref{EQN:A1} and \eqref{EQN:A2} with $\alpha\leq\cO\left(\frac{1}{(\mu_1(\BL)\vee\mu_2(\BL)) r}\right)$, 
respectively. With properly chosen parameters, the output of $\mathrm{AltProj}$, $\BL_k$, satisfies 
\begin{equation*}
    \left\|\BL-\BL_k\right\|_2\leq\eps, \quad \left\|\BS-\BS_k\right\|_\infty\leq\frac{\varepsilon}{\sqrt{mn}} \quad\textnormal{ and }\quad \mathrm{supp}(\BS_k)\subseteq \mathrm{supp}(\BS)
\end{equation*}
in $k=\cO\left(r\log\left(\frac{\left\|\BL\right\|_2}{\eps  }\right)\right)$ iterations.
\end{theorem}

Note the statement of Theorem~\ref{thm: rpca} in the original paper gives a Frobenius norm bound for $\BL-\BL_k$, but the version stated here follows easily from the fact that $\left\|\cdot\right\|_2\leq\left\|\cdot\right\|_F$.

\subsection{CUR}

Error analysis for CUR decompositions has been considered in various places, typically under assumptions on the column selection scheme.  Below is a sample perturbation bound which is independent of the column selection procedure.

\begin{theorem}[{\cite[Remark 3.14]{HH_PCD2019}}]\label{thm: CUR-Bound} 
Let $\BL$ have rank $r$ and compact SVD $\BL=\BW\BSigma\BV^T$; let $\BC=\BL(:,J)$, $\BR=\BL(I,:)$, and $\BU=\BL(I,J)$ for some $I\subseteq[m]$, $J\subseteq[n]$.  Let $\widehat{\BC}\in\R^{m\times|J|}$ and $\widehat{\BR}\in\R^{|I|\times n}$ be arbitrary, with $\widehat{\BU}=\widehat{\BC}(I,:)$. For any Schatten $p$--norm, if $\sigma_r(\BU)=\sigma_{\min}(\BU)\geq12\max\left\{\left\|\widehat{\BR}-\BR\right\|, \left\|\widehat{\BC}-\BC\right\|\right\}$ then the following holds: 
\begin{multline}
    \left\|\BL-\widehat{\BC}\widehat{\BU}_r^\dagger\widehat{\BR}\right\| \leq \\ \left( \frac76\left(\left\|\BW_{\BL}(I,:)^\dagger\right\|+\left\|\BV_{\BL}(J,:)^\dagger\right\|\right)+\frac{25}{6}\left\|\BW_{\BL}(I,:)^\dagger\right\|\left\|\BV_{\BL}(J,:)^\dagger\right\|+\frac{1}{6}\right)\max\left\{\left\|\widehat{\BR}-\BR\right\|,\left\|\widehat{\BC}-\BC\right\|\right\}.
\end{multline}
\end{theorem}

\begin{lemma}\label{COR: Uniform-U-Bound}
Suppose that $\BL$ satisfies \eqref{EQN:A1}. If $J\subseteq[n]$ with $|J|\geq\gamma_1\mu_2(\BL) r$  and $I\subseteq[m]$ with $|I|\geq\gamma_2\mu_1(\BL) r$ for some $\gamma_1,\gamma_2>0$ are chosen by sampling $[n]$ and $[m]$ uniformly without replacement respectively. Let $\BU=\BL(I,J)$. Then with probability at least $1-r\left(\frac{e^{-\delta}}{(1-\delta)^{1-\delta}}\right)^{\gamma_1}-r\left(\frac{e^{-\delta}}{(1-\delta)^{1-\delta}}\right)^{\gamma_2}$
\begin{equation}
   \left\|\BU^\dagger\right\|_2\leq \frac{1}{(1-\delta)\sigma_{\min}(\BL)}\sqrt{\frac{mn}{|I||J|}}.
\end{equation}
\end{lemma}

\begin{proof}
The proof follows by similar argument to the proof of \cite[Lemma 3.12]{HH_PCD2019}.
\end{proof}

\begin{theorem}\label{THM:AppendixRCURError}
Given the notations and assumptions of Theorem~\ref{THM:MainErrorEstimate}, suppose that $\BC$ and $\BR$ satisfy conditions \eqref{EQN:A1} and \eqref{EQN:A2} with sparsity levels 
$\alpha\leq \cO\left(\frac{1}{(\mu_1(\BC)\vee\mu_2(\BC))r}\right)$ and $\cO\left(\frac{1}{(\mu_1(\BR)\vee\mu_2(\BR)r}\right)$, respectively.
Suppose that $\mathrm{AltProj}$ is run for sufficiently many iterations in Lines~\ref{AlgLine:RPCA1} and \ref{AlgLine:RPCA2} of Algorithm~\ref{ALG:Uniform} such that $\max\left\{\left\|\BC-\widehat{\BC}\right\|_2,\left\|\BR-\widehat{\BR}\right\|_2\right\}\leq \varepsilon \frac{(1-\delta)\sigma_{\min}(\BL)}{12}\sqrt{\frac{|I||J|}{mn}}$. Set $\widehat{\BU}=\widehat{\BC}(I,:)$. Then
\begin{equation}\label{rCUR_IEQ}
    \left\|\BL-\widehat{\BC}\widehat{\BU}^\dagger\widehat{\BR}\right\|_2 \leq \frac{7}{8}\varepsilon\sigma_{\min}(\BL) 
\end{equation}
with probability at least 
$1-\frac{r}{n^{c_2(\delta+(1-\delta)\log(1-\delta))}}-\frac{r}{m^{c_1(\delta-(1-\delta)\log(1-\delta))}}.$
\end{theorem}
\begin{proof}
Since $I$ and $J$ are chosen uniformly, by Lemma~\ref{COR: Uniform-U-Bound}  we have that with probability at least 
$1-\frac{r}{n^{c_2(\delta+(1-\delta)\log(1-\delta))}}-\frac{r}{m^{c_1(\delta-(1-\delta)\log(1-\delta))}}$
\[
   \left\|\BU^\dagger\right\|_2\leq \frac{1}{(1-\delta)\sigma_{\min}(\BL)}\sqrt{\frac{mn}{|I||J|}}.
\]
In addition, $\max\left\{\left\|\BC-\widehat{\BC}\right\|_2,\left\|\BR-\widehat{\BR}\right\|_2\right\}\leq \varepsilon \frac{\sigma_{\min}(\BL)(1-\delta)}{12}\sqrt{\frac{|I||J|}{mn}}$ and $\widehat{\BU}=\widehat{\BC}(I,:)$. Combining these estimates and noting that $\left\|\BU^\dagger\right\|_2 = \frac{1}{\sigma_{\min}(\BU)}$, we have
\[
  \sigma_r(\BU)\ge(1-\delta)\sigma_{\min}(\BL)\sqrt{\frac{|I||J|}{mn}}>12\varepsilon \frac{(1-\delta)\sigma_{\min}(\BL)}{12}\sqrt{\frac{|I||J|}{mn}}\geq 12\max\left\{\left\|\widehat{\BC}-\BC\right\|_2,\left\|\widehat{\BR}-\BR\right\|_2\right\},
\]
with probability at least 
$1-\frac{r}{n^{c_2(\delta+(1-\delta)\log(1-\delta))}}-\frac{r}{m^{c_1(\delta-(1-\delta)\log(1-\delta))}}$. The conclusion of the theorem follows from applying Theorem~\ref{thm: CUR-Bound} together with estimates of  $\left\|\BW_{\BL}(I,:)^\dagger\right\|_2$ and $\left\|\BV_{\BL}(J,:)^\dagger\right\|_2$ by Theorem~\ref{THM:UniformBeta}, which implies that with the given success probability,
\[\left\|\left(\BV_{\BL}(J,:)\right)^\dagger\right\|_2\leq \sqrt{\frac{n}{(1-\delta)|J|}},\] and similarly for $\BW_{\BL}(I,:)$ by replacing $n$ and $|J|$ by $m$ and $|I|$, respectively.
\end{proof}

\begin{proof}[Proof of Theorem~\ref{THM:MainErrorEstimate}]
Note that the first terms in the lower bound for $|I|$ and $|J|$ make the conclusions of Theorem~\ref{COR:BetaDeltaBound} and Lemma~\ref{COR: Uniform-U-Bound} valid, while the second terms ensure $2\alpha$--sparsity of $\BS(I,:)$ and $\BS(:,J)$ (Proposition \ref{PROP:Sparsity}).  Consequently, $\BC=\BL(:,J)$ and $\BR=\BL(I,:)$ satisfy condition \eqref{EQN:A1} with the incoherence parameters in Theorem~\ref{THM:BetaBound}.   Additionally, to apply Theorem~\ref{thm: rpca} to $\widetilde{\BC}=\BD(:,J)$ and $\widetilde{\BR}=\BD(I,:)$, we require that $2\alpha\leq \cO(\frac{1}{(\mu_1(\BC)\vee\mu_2(\BC))r)})$ and $2\alpha\leq \cO(\frac{1}{(\mu_1(\BR)\vee\mu_2(\BR))r)})$.  Using the bounds of Theorem~\ref{COR:BetaDeltaBound}, which hold with high probability given how $I$ and $J$ were selected, we see that the assumption that $\alpha=\cO(\frac{1-\delta}{\kappa(\BL)^2(\mu_1(\BL)\vee\mu_2(\BL))})$ guarantees that $\alpha$ is of the required order, and that in turn Theorem~\ref{thm: rpca} holds with the given success probability. 

Finally, Theorem~\ref{THM:AppendixRCURError}  holds on account of the above observations and the fact that the iteration count being $\cO\left(r\log\left(\sqrt{\frac{mn}{|I||J|}}\frac{\kappa(\BL)}{(1-\delta)\eps}\right)\right)$ implies  the bounds on $\left\|\BC-\widehat{\BC}\right\|_2$ and $\left\|\BR-\widehat{\BR}\right\|_2$. Hence, we conclude that the outputs $\widehat{\BC},\widehat{\BR}$ of Algorithm~\ref{ALG:Uniform} satisfy
\[ \left\|\BL-\widehat{\BL}\right\|_2=\left\|\BL-\widehat{\BC}(\widehat{\BC}(I,:))^\dagger\widehat{\BR}\right\|_F\leq\frac{7}{8} \eps\sigma_{\min}(\BL),\]
with the given success probability, which completes the proof.

\end{proof}

\section{Uniform$+$Deterministic: Proof of Theorem~\ref{THM:Hybrid}}\label{APP:Hybrid}

The deterministic greedy column selection algorithm of \cite{AB2013} is reproduced as Algorithm~\ref{ALG:Greedy}. Here, we provide the proof of Theorem~\ref{THM:Hybrid}, beginning with a bound on the pseudoinverse of $\BV_{\BL}(:,J)$ assuming that the column submatrix of $\BL$ and $\BD$ are not too far apart.  We will find use for this result later as well when we analyze the purely deterministic column sampling method.

\begin{algorithm}[h!]
 \caption{A deterministic greedy removal algorithm for subset selection \cite{AB2013}}\label{ALG:Greedy}
\begin{algorithmic}[1]
\STATE \textbf{Input: }{$\BX\in\mathbb{R}^{r\times m}$ with $\rank(\RV{\BX})=r$; $k$: sampling size. }

\STATE\textbf{Output: }{$\mathcal{S}\subseteq [m]$: set of cardinality $k$.}

\STATE \textbf{Initialization: }{ $\mathcal{S}_0:=[m]$}
 
\STATE{ Compute the SVD of $\BX_{\mathcal{S}_0}: \BX_{\mathcal{S}_0}=\BU^{(0)}\BSigma^{(0)}\bm{Y}^{(0)}$}
 
\FOR {$i=1,2,\cdots,m-k$}
 
\STATE{~Let the singular values of $\BX_{\mathcal{S}_{i-1}}$ be $\sigma_1^{(i-1)}, \cdots, \sigma_r^{(i-1)}$.}
  
\STATE{~Let the columns of $\bm{Y}^{(i-1)}$ be $\{y_{k}^{(i-1)}\}_{k\in \mathcal{S}_{i-1}}$. Denote by $y_{kj}$ be the $j-$th element of $y_{k}^{(i-1)}$.}
  
\STATE  {~$k_{i}:=\argmin\limits_{k\in \mathcal{S}_{i-1},\left\|y_{k}^{(i-1)}\right\|_2<1}\left(\frac{\sum_{j=1}^{r}(y_{kj}^{(i-1)}/ \sigma_{k}^{(i-1)} )^2 }{1-\left\|y_{k}^{(i-1)}\right\|_2^2} \right)$.}
  
\STATE  {Set $\mathcal{S}_{i}:=\mathcal{S}_{i-1}\setminus\{k_i\}$}
  
\STATE {~Downdate the SVD of $\BX_{\mathcal{S}_{i-1}}$ to obtain an SVD of $\BX_{\mathcal{S}_i}=\BU^{(i)}\BSigma^{(i)}\bm{Y}^{(i)}$ (see \cite{GE1995})}

\ENDFOR

\end{algorithmic}
\end{algorithm}

Recall that Theorem~\ref{THM:Hybrid} analyzes the error of approximation for the hybrid procedure of uniformly sampling column and row submatrices of $\BD=\BL+\BS$, running RPCA on these to output $\widehat{\BC}$ and $\widehat{\BR}$, and finally using Algorithm~\ref{ALG:Greedy} to select exactly $r=\rank(\BL)$ columns and rows of $\widehat{\BC}$ and $\widehat{\BR}$, respectively to give a compact approximation of $\BL$.

\begin{theorem}\label{thm: bpnd2bpnfd}
Let ${\BL}\in\mathbb{R}^{m\times n}$ with $\rank({\BL})=r$. Let $\widetilde{{\BL}}={\BL}+\bm E$ with compact SVD $\widetilde{\BL}=\widetilde{{\BW}}\widetilde{\BSigma}\widetilde{{\BV}}^T$, and suppose $\left\|\bm E\right\|_2\leq\frac{1}{2}\sigma_{\min}({\BL})$. Suppose that ${I}\subseteq[m]$ and $J\subseteq[n]$ with $|I|=|J|=r$ are obtained by applying Algorithm~\ref{ALG:Greedy} on $\widetilde{{\BW}}$ and $\widetilde{{\BV}}$. Suppose that \[\left\|\BL(:,J)-\widetilde{\BL}(:,J)\right\|_2=\left\|{\BW_{\BL}}\BSigma_{\BL} ({\BV_{\BL}}(J,:))^T-\widetilde{{\BW}}\widetilde{\BSigma}(\widetilde{{\BV}}(J,:))^T\right\|_2\leq \frac{1}{8}\sqrt{\frac{|J|}{rn}}\sigma_{\min}({\BL}).\]
Then $\left\|\left({\BV_{\BL}}(J,:)\right)^\dagger\right\|_2\leq 4\kappa({\BL})\sqrt{\frac{rn}{|J|}}$.
\end{theorem}

\begin{proof}
Set $\widetilde{\BC}=\widetilde{{\BW}}\widetilde{\BSigma}(\widetilde{{\BV}}(J,:))^T$ and ${\BC}={\BL}(:,J)={\BW_{\BL}}\BSigma_{\BL} ({\BV_{\BL}}(J,:))^T$. Then we have 
\begin{equation}\label{eqn:VJdagger}
\left\|\left({\BV_{\BL}}(J,:)\right)^\dagger\right\|_2=\left\|{\BC}^\dagger {\BW_{\BL}}\BSigma_{\BL}\right\|_2\leq \sigma_{\max}({\BL})\left\|{\BC}^\dagger\right\|_2.
\end{equation}

Notice that $\rank({\BC})=\rank(\widetilde{\BC})=r$. By \cite[Theorems~3.1--3.4]{Stewart_1977}, we have \[
\left\|{\BC}^\dagger-\widetilde{\BC}^\dagger\right\|_2\leq 2\left\|{\BC}^\dagger\right\|_2\left\|\widetilde{\BC}^\dagger\right\|_2\left\|{\BC}-\widetilde{\BC}\right\|_2.
\]

In addition, by Proposition \ref{COR:MaxVolBetaBound},
\begin{eqnarray}
\left\|\widetilde{\BC}^\dagger\right\|_2 &=&\left\| \left((\widetilde{{\BV}}(J,:)  )^T\right)^\dagger\widetilde{\BSigma}^\dagger\widetilde{{\BW}}^T\right\|_2\nonumber\\
&\leq& \left\|\left(\widetilde{{\BV}}(J,:)\right)^\dagger\right\|_2\left\|\widetilde{\BSigma}^\dagger \right\|_2\nonumber\\
&\leq& \sqrt{\frac{rn}{|J|}}\frac{1}{\sigma_{\min}(\widetilde{{\BL}})}\nonumber\\
&\leq& \sqrt{\frac{rn}{|J|}}\frac{1}{\sigma_{\min}({\BL})-\left\|\BE\right\|_2}\nonumber\\
&\leq& 2\sqrt{\frac{rn}{|J|}}\frac{1}{\sigma_{\min}({\BL})}. \label{eqn:Ddagger}
\end{eqnarray}
Since $\left\|\BW_{\BL}\BSigma_{\BL} ({\BV}_{\BL}(J,:))^T-\widetilde{{\BW}}\widetilde{\BSigma}(\widetilde{{\BV}}(J,:))^T\right\|_2\leq \frac{1}{8}\sqrt{\frac{|J|}{rn}}\sigma_{\min}({\BL})$,  $\left\|{\BC}^\dagger\right\|_2\leq \frac{ \left\|\widetilde{\BC}^\dagger\right\|_2}{1-2\left\|\widetilde{\BC}^\dagger\right\|_2\left\|{\BC}-\widetilde{\BC}\right\|_2}\leq 2\left\|\widetilde{\BC}^\dagger\right\|_2$.
Combining, \eqref{eqn:VJdagger} and \eqref{eqn:Ddagger}, we have 
\begin{equation}
    \left\|{\BV}_{\BL}(J,:)^\dagger\right\|\leq \sqrt{\frac{rn}{|J|}}\frac{4\sigma_{\max}({\BL})}{\sigma_{\min}({\BL})}=4\kappa({\BL})\sqrt{\frac{rn}{|J|}}.
\end{equation}
\end{proof}

\begin{proof}[Proof of Theorem~\ref{THM:Hybrid}]
The assumptions of Theorem~\ref{THM:MainErrorEstimate} imply that $\widetilde{\BC}=\BD(:,J)=\BL(:,J)+\BS(:,J)$ and $\widetilde{\BR}=\BD(I,:)=\BL(I,:)+\BS(I,:)$ satisfy the necessary incoherence and sparsity conditions for RPCA to be successful.  Thus we pick up most of the way through the proof of Theorem~\ref{THM:MainErrorEstimate} above, and note that the new assumption on the iteration count for $\mathrm{AltProj}$ in the statement of this theorem implies by Theorem~\ref{thm: rpca} that we have
\begin{equation}\label{eqn:From_uniform_smp}
\max\left\{\left\|\BC-\widehat{\BC}\right\|_2,\left\|\BR-\widehat{\BR}\right\|_2\right\}\leq  \frac{\varepsilon(1-\delta)^2\sigma_{\min}(\BL)}{ 229\kappa(\BL)^4r\sqrt{mn\mu_1(\BL)\mu_2(\BL)}}.
\end{equation}
According to Theorem~\ref{THM:UniformBeta}, the following statements hold:
\begin{equation} \label{eqn: upper_vl_psinv}
\begin{split}
\left\|\BW_{\BL}(I,:)^\dagger\right\|_2&\leq \sqrt{\frac{m}{(1-\delta)|{I}|}}\\
\left\|\BV_{\BL}(J,:)^\dagger\right\|_2&\leq \sqrt{\frac{n}{(1-\delta)|{J}|}},
\end{split}
\end{equation}
with probability at least $1-\frac{r}{n^{c_2(\delta+(1-\delta)\log(1-\delta))}}-\frac{r}{m^{c_1(\delta-(1-\delta)\log(1-\delta))}}$. 
From \eqref{eqn:From_uniform_smp}, we have 
\begin{eqnarray*}
\left\|\widehat{\BC}_1-\BC(:,{J}_1)\right\|_2&\leq& \frac{\varepsilon(1-\delta)^2\sigma_{\min}(\BL)}{ 229\kappa(\BL)^4r\sqrt{mn\mu_1(\BL)\mu_2(\BL)}},\\
\left\|\widehat{\BR}_1-\BR({I}_1,:)\right\|_2&\leq&\frac{\varepsilon(1-\delta)^2\sigma_{\min}(\BL)}{ 229\kappa(\BL)^4r\sqrt{mn\mu_1(\BL)\mu_2(\BL)}}.
\end{eqnarray*}
In addition, 
    \begin{align*}
        \frac{1}{8}\sqrt{\frac{|J_1|}{r|J|}}\sigma_{\min}(\BC)=\frac{1}{8\sqrt{|J|}}\frac{1}{\left\|\BC^\dagger\right\|_2}\geq \frac{1}{8\sqrt{|J|}}\frac{1}{\left\|\BL^\dagger\right\|_2\left\|\left(\BV_{\BL}(J,:)\right)^\dagger\right\|_2} \geq \frac{\sqrt{1-\delta}\sigma_{\min}(\BL)}{8\sqrt{n}}.
    \end{align*}
    Thus, $\left\|\widehat{\BC}_1-\BC(:,{J}_1) \right\|_2\leq \frac{1}{8}\sqrt{\frac{|J_1|}{r|J|}}\sigma_{\min}(\BC)$. Similarly, we have  $\left\|\widehat{\BR}_1-\BR({I}_1,:)\right\|_2\leq \frac{1}{8}\sqrt{\frac{|I_1|}{r|I|}}\sigma_{\min}(\BR)$. 
By Theorem~\ref{thm: bpnd2bpnfd}, 
\begin{eqnarray*}
\left \|{\BV}_{\BC}({J}_1,:)^\dagger\right\|&\leq&4\kappa({\BC})\sqrt{ |J|},\\
\left \|{\BW}_{\BR}({I}_1,:)^\dagger\right\|&\leq& 4\kappa({\BR})\sqrt{|I|}.
\end{eqnarray*}
Since $\BC(:,{J}_1)=\BW_{\BC}\BSigma_{\BC}(\BV_{\BC}({J}_1,:))^T=\BW_{\BL}\BSigma_{\BL}(\BV_{\BL}(\tilde{J}_1,:))^T$,  we have
\begin{eqnarray*}
\left\|\BV_{\BL}(\tilde{J}_1,:)^\dagger\right\|_2&\leq&\left\| \BV_{\BC}({J}_1,:)^\dagger\right\|_2\left\|\BSigma_{\BL}\right\|_2\left\|\BSigma_{\BC}^{-1}\right\|_2\\
&\stackrel{ Lemma~ \ref{LEM:CoverA}}=&4\kappa(\BC)\sqrt{|J|}\left\|\BL\right\|_2\left\|\BL^\dagger\right\|_2\left\|\BV_{\BL}(J,:)^\dagger\right\|_2\label{eqn:9.3}\\
&\stackrel{Lemma~\ref{LEM:kCkAbound}}\leq&4\kappa(\BL)^2\frac{|J|\sqrt{\mu_2r}}{\sqrt{n}}\left\|\BV_{\BL}(J,:)^\dagger\right\|_2^2\label{eqn:9.4}\\
&\stackrel{ \eqref{eqn: upper_vl_psinv}}\leq& 4\kappa(\BL)^2\frac{ \sqrt{\mu_2(\BL)rn}}{1-\delta}.
\end{eqnarray*}
Similarly, we have
\begin{equation*}
  \left\|\BW_{\BL}({I}_1,:)^\dagger\right\|\leq  4\kappa(\BL)^2\frac{ \sqrt{\mu_1(\BL)rm}}{1-\delta}.
\end{equation*}
Thus ,
\begin{equation*}
    \left\|\BU^\dagger \right\|_2\leq \frac{16\kappa^4(\BL)r\sqrt{\mu_1(\BL)\mu_2(\BL)mn}}{(1-\delta)^2\sigma_{\min}(\BL)},
\end{equation*}
i.e., $\sigma_r(\BU)\geq \frac{(1-\delta)^2\sigma_{\min}(\BL)}{16\kappa^4(\BL)r\sqrt{\mu_1(\BL)\mu_2(\BL)mn}}\geq 12\max\left\{ \left\|\widehat{\BC}_1-\BL(:,\tilde{J}_1)\right\|_2, \left\|\widehat{\BR}_1-\BL(\tilde{I}_1,:)\right\|_2\right\}$. Using the bounds of $\left\|\left(\BW(I,:)\right)^\dagger\right\|_2$, $\left\|\left(\BV(J,:)\right)^\dagger\right\|_2$ and Theorem~\ref{thm: CUR-Bound},  we obtain 
\begin{equation*}
     \left\|\BL-\widehat{\BC}_1\widehat{\BU}_1^\dagger\widehat{\BR}_1\right\|\leq\frac{1}{3}\varepsilon\sigma_r(\BL)
\end{equation*}
after some simple calculations.
\end{proof}

\section{Deterministic Column and Row Selection}\label{APP:Det1}

In this section, we study what happens when we apply the greedy algorithm (Algorithm~\ref{ALG:Greedy}) on the singular vectors of $\BD=\BL+\BS$ to choose column and row submatrices.  We primarily consider how the incoherence of $\BC=\BL(:,J)$ relates to that of $\BL$.  Our main result is that, if we first initialize $\BD$ by one step of Alternating Projections \cite{netrapalli2014non} with a judiciously chosen hard thresholding parameter, then the resulting approximation $\BL_0$ to $\BL$ is good enough to ensure that the incoherence does not inflate significantly.  To achieve this, we study conditions under which the hypothesis of Theorem~\ref{THM:AppendixRCURError} holds, namely such that $\left\|\BL(:,J)-\BL_0(:,J)\right\|_2\leq\frac{1}{8}\sqrt{\frac{|J|}{rn}}\sigma_{\min}(\BL)$, where $\BL_0$ is the output of Algorithm~\ref{ALG:Init}.  To proceed, we first state the initialization required.

\begin{algorithm}[h!]
 \caption{Initialization by One Step of Alternating Projections}\label{ALG:Init}
\begin{algorithmic}[1]
\STATE \textbf{Input: }{$\BD=\BL+\BS$: matrix to be split; $r$: rank of $\BL$; $\eta$: thresholding parameter. }

\STATE\textbf{Output: }{$\BL_0,\BS_0$}

\STATE{$\zeta_{0} = \eta \cdot \sigma_1(\BD)$}

\STATE{$\BS_{0}=\mathcal{T}_{\zeta_{0}}(\BD)$  \qquad\qquad\qquad\qquad\qquad  $\rhd~\mathcal{T_\zeta}$: hard thresholding operation with thresholding value $\zeta$}

\STATE{$\BL_0=\mathcal{D}_r(\BD-\BS_{0})$  \qquad\qquad\qquad\qquad  $\rhd~\mathcal{D}_r$: truncated rank-$r$ SVD operator}
 
\end{algorithmic}
\end{algorithm}

Here, $\mathcal{T}_{\zeta_0}(\BD)$ is the hard thresholding operator, whose output is (entrywise) $\BD_{ij}$ if $\BD_{ij}\geq\zeta_0$ and $0$ otherwise.  Given a matrix $\bm{A}$, $\mathcal{D}_r(\bm{A})$ is its projection onto the set of rank $r$ matrices (e.g., $\mathcal{D}_r(\bm{A})=\BU_{\bm{A},r}\BSigma_{\bm{A},r}\BV_{\bm{A},r}^T$).

\begin{theorem}
\label{thm:initialization bound} Let $\BL\in\mathbb{R}^{m\times n}$ be a rank $r$ $\left\{\mu_1(\BL),\mu_2(\BL)\right\}$-incoherent matrix and $\BS\in\mathbb{R}^{m\times n}$ be an $\alpha$-sparse matrix. 
Let $\mu:=\max\left\{\mu_1(\BL),\mu_2(\BL)\right\}$. 
If the sparsity level satisfies  $\alpha\leq \frac{1}{256\mu^{3/2} r^2 \kappa(\BL)}$ and the thresholding parameter in Algorithm~\ref{ALG:Init} obeys $\frac{\mu r\sigma_{\max}(\BL)}{\sqrt{mn}\sigma_{\max}(\BD)}\leq\eta\leq\frac{3\mu r\sigma_{\max}(\BL)}{\sqrt{mn}\sigma_{\max}(\BD)}$, then the outputs of  Algorithm~\ref{ALG:Init} satisfy
\begin{align*}
\left\|\BL-\BL_0\right\|_2 & \leq 8\alpha\mu r \sigma_{\max}(\BL),\\ \max_i\left\|(\BL_0-\BL)^T\bm{e}_i\right\|_2 &\leq \frac{32\alpha \mu^{1.5} r^{1.5}}{\sqrt{m}}\sigma_{\max}(\BL),\\  
\max_j\left\|(\BL_0-\BL)\bm{e}_j\right\|_2 & \leq \frac{32\alpha \mu^{1.5} r^{1.5}}{\sqrt{n}}\sigma_{\max}(\BL), \\ 
\left\|\BS-\BS_0\right\|_\infty &\leq \frac{\mu r}{n} \sigma_{\max}(\BL),\quad \textnormal{and} \quad
\supp(\BS_0)\subseteq \supp(\BS).
\end{align*}
\end{theorem}

Note that since $\left\|\bm{A}(I,:)\right\|_F^2\leq\sum_{i\in I} \left\|\bm{A}^T \bm{e}_i\right\|_2^2$ and $\left\|\bm{A}(:,J)\right\|_F^2\leq\sum_{j\in J} \left\|\bm{A} \bm{e}_j\right\|_2^2$, Theorem \ref{thm:initialization bound} implies
\begin{align*}
        \left\|\BL(I,:)-\BL_0(I,:)\right\|_F  &\leq \frac{32\alpha \mu^{1.5} r^{1.5}\sqrt{|I|}}{\sqrt{m}}\sigma_{\max}(\BL)\leq \frac{1}{8}\sqrt{\frac{|I|}{rm}}\sigma_{\min}(\BL), \cr 
    \left\|\BL(:,J)-\BL_0(:,J)\right\|_F  &\leq \frac{32\alpha \mu^{1.5} r^{1.5}\sqrt{|J|}}{\sqrt{n}}\sigma_{\max}(\BL)\leq \frac{1}{8}\sqrt{\frac{|J|}{rn}}\sigma_{\min}(\BL), \cr
      \left\|\BL-\BL_0\right\|_F  &\leq \sqrt{2r}\left\|\BL-\BL_0\right\|_2 \leq  \frac{1}{2}\sigma_{\min}(\BL),
\end{align*}
where the second parts of the inequalities follow from bound of $\alpha$.

\begin{remark}\label{REM:GreedyBeta}
Note that the conclusion of Theorem \ref{thm:initialization bound} implies that $\BL_0$ satisfies the necessary requirement for Theorem \ref{thm: bpnd2bpnfd}.  Consequently, we find that the output of Algorithm~\ref{ALG:Greedy} applied to $\BL_0$ yields a column submatrix $\BC=\BL(:,J)$ with parameter $\beta\leq 4\kappa(\BL)\sqrt{r}$, and thus incoherence $\mu_2(\BC)\leq 16\kappa(\BL)^2r$ according to Theorem \ref{THM:BetaBound}.  This bound for $\beta$ is a factor of $\sqrt{r}$ better than simply applying the greedy algorithm directly to $\BV_{\BL}$ due to good initialization to find $\BL_0$.  This method also has the benefit of being tractable as it does not require knowledge of $\BV_{\BL}$.
\end{remark}

\section{Proof of Theorem~\ref{thm:initialization bound}}\label{APP:Det2}

As shown in prior arts \cite{cai2018thesis, cai2019accelerated, cai2021accelerated, netrapalli2014non}, the convergence analysis of the alternating projections based RPCA algorithms can be reduced to the case of symmetric matrices, since non-symmetric matrix recovery problems can be cast as problems with respect to symmetric augmented matrices. For more details about how to reduce the general RPCA problems to symmetric cases, we refer the interested reader to \cite[Section~3.2]{cai2018thesis}. Firstly, We shall present two technical lemmata in the symmetric setting.

\begin{lemma}[{\cite[Lemma~4]{netrapalli2014non}}]  \label{lemma:bound of sparse matrix}
Let $\BS \in \mathbb{R}^{n\times n}$ be an $\alpha$-sparse symmetric matrix. Then, the following inequality holds:
\[
\left\|\BS\right\|_2 \leq \alpha n\left\|\BS\right\|_\infty.
\]
\end{lemma}

\begin{lemma}[{\cite[Lemma~5]{netrapalli2014non}}] \label{init:lemma:bound_power_vector_norm_with_incoherence}
Let $\BS\in\mathbb{R}^{n\times n}$ be an $\alpha$-sparse symmetric matrix. Let $\BW\in\mathbb{R}^{n\times r}$ be an orthogonal matrix with $\mu$-incoherence, i.e., $\left\|\BW^T\bm{e}_i\right\|_2\leq\sqrt{\frac{\mu r}{n}}$ for all $i$. Then
\[
\left\|\BW^T\BS^a\bm{e}_i\right\|_2\leq \sqrt{\frac{\mu r}{n}}\left(\alpha n\left\|\BS\right\|_\infty\right)^a
\]
holds for all $i$ and $a\geq 0$.
\end{lemma}

We are now ready to prove Theorem~\ref{thm:initialization bound}.

\begin{proof}[Proof of Theorem~\ref{thm:initialization bound}]
As discussed above, we only need to prove the theorem under symmetric setting, and it can be generalized to non-symmetric matrices. The proof consists of four parts. The first three parts have been shown in the proof of \cite[Theorem~2]{cai2019accelerated}, we still include them here for the completeness.

\noindent\textbf{Part 1: }
First, note that 
\begin{equation*}
\left\|\BL\right\|_\infty = \max_{ij} |\bm{e}_i^T\BW_{\BL} \BSigma_{\BL} \BV_{\BL}^T \bm{e}_j| \leq \max_{ij} \left\|\BW_{\BL}^T\bm{e}_i\right\|_2\left\|\BSigma_{\BL}\right\|_2\left\|\BV_{\BL}^T\bm{e}_j\right\|_2\leq\frac{\mu r}{n}\sigma_1(\BL)
\end{equation*}
where the last inequality follows from $\mu$-incoherence of $\BL$. 
Since $\eta\geq\frac{\mu r\sigma_1(\BL)}{n\sigma_1(\BD)}$, we get 
\begin{equation} \label{eq:Init:step 1 result}
\left\|\BL\right\|_\infty\leq \eta\sigma_1(\BD) =: \zeta_{0}.
\end{equation}
Considering the entries of $\BS_0$, we notice
\[
[\BS_0]_{ij}=[\mathcal{T}_{\zeta_{0}} (\BS+\BL)]_{ij} =
\begin{cases}
\mathcal{T}_{\zeta_{0}} ([\BS+\BL]_{ij}) & (i,j)\in\supp(\BS) \cr
\mathcal{T}_{\zeta_{0}} ([\BL]_{ij})        & (i,j)\not\in\supp(\BS)  \cr
\end{cases}
\]
since $[\BS]_{ij}=0$ outside of its support. Together with \eqref{eq:Init:step 1 result}, we have $[\BS_0]_{ij}=0$ for all $(i,j)\not\in\supp(\BS)$, which implies $\supp(\BS_0)\subseteq\supp(\BS)$. Furthermore, we have
\begin{align*}
[\BS-\BS_{0}]_{ij} =
\begin{cases}
0 &  \cr
[\BL]_{ij} &   \cr
[\BS]_{ij}         & \cr
\end{cases} \leq
\begin{cases}
0 &  \cr
\left\|\BL\right\|_\infty       &  \cr
\left\|\BL\right\|_\infty +\zeta_{0} & \cr
\end{cases} \leq
\begin{cases}
0 &  (i,j)\not\in \supp(\BS)   \cr
\frac{\mu r}{n} \sigma_1(\BL)       & (i,j)\in \supp(\BS_0)   \cr
\frac{4\mu r}{n}  \sigma_1(\BL) & (i,j)\in \supp(\BS)\setminus\supp(\BS_0) \cr
\end{cases}
\end{align*}
where the last inequality follows from $\eta\leq\frac{3\mu r\sigma_1(\BL)}{n\sigma_1(\BD)}$, which implies $\zeta_{0}\leq \frac{3\mu r}{n}\sigma_1(\BL)$. Overall, we get
\begin{equation}\label{eq:SminusS}
\supp(\BS_{0})\subseteq \supp(\BS) \quad \textnormal{and}\quad \left\|\BS-\BS_{0}\right\|_\infty\leq\frac{4\mu r}{n}\sigma_1(\BL).
\end{equation}
Moreover, this implies that $\BS-\BS_0$ is also $\alpha$-sparse.

\noindent\textbf{Part 2: } 
Since $\BL_0=\mathcal{D}_r(\BS-\BS_{0})$ is the best rank $r$ approximation of $\BD-\BS_{0}$, so we have
\begin{align*}
    \left\|\BL-\BL_0\right\|_2&\leq\left\|\BL-(\BD-\BS_{0})\right\|_2 + \left\|(\BD-\BS_{0})-\BL_0\right\|_2 \cr
           &\leq 2\left\|\BL-(\BD-\BS_{0})\right\|_2\cr
           &= 2\left\|\BL-(\BL+\BS-\BS_{0})\right\|_2\cr
           &= 2\left\|\BS-\BS_{0}\right\|_2 \cr
           &\leq 2 \alpha n  \left\|\BS-\BS_{0}\right\|_\infty
\end{align*}
where the last inequality uses Lemma~\ref{lemma:bound of sparse matrix}. By applying \eqref{eq:SminusS}, we have
\begin{equation}\label{eq:norm:L-L0}
\left\|\BL-\BL_0\right\|_2\leq 8\alpha \mu r\sigma_1(\BL)  .
\end{equation}

\noindent\textbf{Part 3: }
Let $\lambda_i$ denote the $i^{th}$ eigenvalue of $\BD-\BS_{0}$ ordered as $|\lambda_1|\geq|\lambda_2|\geq\cdots\geq|\lambda_n|$. 
Since $\BD-\BS_{0}=\BL+(\BS-\BS_{0})$, by  Weyl's inequality, we have
\begin{equation}   \label{init:eq:eigenvalues bound 0}
\big|\sigma_i(\BL)-|\lambda_i|\big| \leq \left\|\BS-\BS_{0}\right\|_2 \leq \alpha n \left\|\BS-\BS_{0}\right\|_\infty\leq \frac{\sigma_r(\BL)}{8}
\end{equation}
for all $i$, where the last inequality follows from the fact that $\alpha\leq\frac{1}{32\mu r \kappa(\BL)}$. Therefore, we have
\begin{align} 
&\frac{7}{8}\sigma_i(\BL) \leq |\lambda_i| \leq \frac{9}{8}\sigma_i(\BL),\qquad  1\leq i\leq r,  \label{init:eq:eigenvalues bound 1}
\end{align}
and
\begin{align}
&\frac{\left\|\BS-\BS_{0}\right\|_2}{|\lambda_r|}\leq \frac{\frac{\sigma_r(\BL)}{8}}{\frac{7\sigma_r(\BL)}{8}}=\frac{1}{7}.\label{init:eq:eigenvalues bound 2}
\end{align}

\noindent\textbf{Part 4: }
Let $\BD-\BS_{0}=\left[\BW_0, \ddot{\BW}_0 \right]\begin{bmatrix}\bm{\Lambda}  & \bm{0}\\\bm{0} &\ddot{\bm{\Lambda}}\end{bmatrix}\left[\BW_0, \ddot{\BW}_0\right]^T =\BW_0\bm{\Lambda} \BW_0^T+\ddot{\BW}_0\ddot{\bm{\Lambda}}\ddot{\BW}_0^T$ be the eigenvalue decomposition of $\BD-\BS_{0}$, where the eigenvalues are sorted by magnitude and $\bm{\Lambda} $ consist of the $r$ largest ones while $\ddot{\bm{\Lambda}}$ contains the rest. Correspondingly, $\BW_0$ contains the first $r$ eigenvectors, and $\ddot{\BW}_0$ has the rest. 
By the symmetric setting, we have $\BL_0=\mathcal{D}_r(\BD-\BS_{0})=\BW_0\bm{\Lambda} \BW_0^T$.
Denote $\BE=\BS-\BS_{0}$ and $\BW(:,i)$ to be the $i$--th eigenvector of $\BD-\BS_{0}=\BL+\BE$. So we can see 
\begin{equation*}
    (\BL+\BE)\BW(:,i) =  \lambda_i \BW(:,i)
\end{equation*}
for $1\leq i\leq r$. Hence,
\begin{align}  \label{eq:eigenvactor_series}
\BW(:,i) =   \left(I-\frac{\BE}{\lambda_i}\right)^{-1}\frac{\BL}{\lambda_i}\BW(:,i)=
\sum_{j=0}^\infty\left(\frac{\BE}{\lambda_i}\right)^j
\frac{\BL}{\lambda_i}\BW(:,i)
\end{align}
for all $1\leq i\leq r$. Note that the expansion in the last equality is valid since \eqref{init:eq:eigenvalues bound 2} implies $ \frac{\left\|\BE\right\|_2}{|\lambda_i|}<\frac{1}{7}\leq 1$ for all $1\leq i\leq r$. 

We will first prove the \textit{row version} of the inequality.
By applying \eqref{eq:eigenvactor_series}, we get 
\begin{align*}
&~\max_i\left\|(\BL_0-\BL)^T\bm{e}_i\right\|_2 \cr
=&~\max_i\left\|(\BW_0\bm{\Lambda} \BW_0^T -\BL)^T\bm{e}_i\right\|_2 \cr
=&~\max_i\left\|\big(\BL\BW_0\bm{\Lambda} ^{-1}\BW_0^T\BL-\BL  + \sum_{a+b>0} \BE^a\BL\BW_0\bm{\Lambda} ^{-(a+b+1)}\BW_0^T\BL\BE^b \big)^T\bm{e}_i\right\|_2  \cr
\leq&~\max_i\left\|(\BL\BW_0\bm{\Lambda} ^{-1}\BW_0^T\BL-\BL)^T\bm{e}_i\right\|_2  + \sum_{a+b>0} \max_i \left\|\left(\BE^a\BL\BW_0\bm{\Lambda} ^{-(a+b+1)}\BW_0^T\BL\BE^b \right)^T\bm{e}_i \right\|_2 \cr
:=&~ \bm{Y}_0 + \sum_{a+b>0} \bm{Y}_{ab}.
\end{align*}

We will first bound $\bm{Y}_0$. By the incoherence property of $\BL$, i.e., $\max_i \left\|\BW\BW^T\bm{e}_i\right\|_2\leq\sqrt{\frac{\mu r}{n}}$, we have
\begin{align*}
\bm{Y}_0 &= \max_{i} \left\|(\BL\BW_0\bm{\Lambda} ^{-1}\BW_0^T\BL-\BL)^T\bm{e}_i\right\|_2 \cr
&=\max_{i} \left\|(\BL\BW_0\bm{\Lambda} ^{-1}\BW_0^T\BL-\BL)^T\BW\BW^T\bm{e}_i\right\|_2 \cr
&\leq\max_{i} \left\|\BW\BW^T\bm{e}_i\|_2~\|\BL\BW_0\bm{\Lambda} ^{-1}\BW_0^T\BL-\BL\right\|_2 \cr		
&\leq \sqrt{\frac{\mu r}{n}}\left\|\BL\BW_0\bm{\Lambda} ^{-1}\BW_0^T\BL-\BL\right\|_2
\end{align*}
where the second equation follows from the fact $\BL=\BW\BW^T\BL$. 

Note that $\BL=\BW_0\bm{\Lambda} \BW_0^T+\ddot{\BW}_0\ddot{\bm{\Lambda}}\ddot{\BW}_0^T -\BE$, so we get 
\begin{align*}
&~\left\|\BL\BW_0\bm{\Lambda} ^{-1}\BW_0^T\BL-\BL\right\|_2 \cr
=&~ \left\|(\BW_0\bm{\Lambda} \BW_0^T+\ddot{\BW}_0\ddot{\bm{\Lambda}}\ddot{\BW}_0^T -\BE)\BW_0\bm{\Lambda} ^{-1}\BW_0^T(\BW_0\bm{\Lambda} \BW_0^T+\ddot{\BW}_0\ddot{\bm{\Lambda}}\ddot{\BW}_0^T -\BE)-\BL\right\|_2  \cr
=&~\left\|\BW_0\bm{\Lambda} \BW_0^T-\BL-\BW_0\BW_0^T\BE-\BE\BW_0\BW_0^T-\BE\BW_0\bm{\Lambda} ^{-1}\BW_0^T\BE\right\|_2  \cr
\leq&~ \left\|\BE-\ddot{\BW}_0\ddot{\bm{\Lambda}}\ddot{\BW}_0^T\right\|_2 +2\left\|\BE\right\|_2+\frac{\left\|\BE\right\|_2^2}{|\lambda_r|}  \cr
\leq&~ \left\|\ddot{\BW}_0\ddot{\bm{\Lambda}}\ddot{\BW}_0^T\right\|_2 +4\left\|\BE\right\|_2  \cr
\leq&~  |\lambda_{r+1}|+4\left\|\BE\right\|_2 \cr
\leq&~  5\left\|\BE\right\|_2
\end{align*}
where the first and fourth inequality follow from (\ref{init:eq:eigenvalues bound 0}) and (\ref{init:eq:eigenvalues bound 2}), and $|\lambda_{r+1}|\leq\left\|\BE\right\|_2$ since $\sigma_{r+1}(\BL)=0$. Together, we have
\begin{equation}  \label{init:eq:Y0 bound}
\bm{Y}_0\leq 5\sqrt{\frac{\mu r}{n}} \left\|\BE\right\|_2 \leq 5\alpha\sqrt{ \mu r n} \left\|\BE\right\|_\infty
\end{equation}
where the last inequality follows from Lemma~\ref{lemma:bound of sparse matrix}.

Next, we will bound $\bm{Y}_{ab}$. Note that
\begin{align*}
\bm{Y}_{ab}&=\max_{i} \left\|(\BE^a\BL\BW_0\bm{\Lambda} ^{-(a+b+1)}\BW_0^T\BL\BE^b)^T\bm{e}_i\right\|_2  \cr
      &=\max_{i} \left\|\BE^b\BL\BW_0\bm{\Lambda} ^{-(a+b+1)}\BW_0^T\BL(\BW\BW^T\BE^a\bm{e}_i)\right\|_2  \cr
      &\leq \max_{i} \left\|\BW^T\BE^a\bm{e}_i\right\|_2~\left\|\BL\BW_0\bm{\Lambda} ^{-(a+b+1)}\BW_0^T\BL\right\|_2~\left\|\BE\right\|_2^b  \cr
      &\leq \sqrt{\frac{\mu r}{n}}( \alpha n\left\|\BE\right\|_\infty)^{a+b} \left\|\BL\BW_0\bm{\Lambda} ^{-(a+b+1)}\BW_0^T\BL\right\|_2 \cr
      &\leq \alpha \sqrt{\mu r n} \left\|\BE\right\|_\infty \left(\frac{\sigma_r(\BL)}{8}\right)^{a+b-1} \left\|\BL\BW_0\bm{\Lambda} ^{-(a+b+1)}\BW_0^T\BL\right\|_2 
\end{align*}
where the second inequality uses Lemmata~\ref{lemma:bound of sparse matrix} and \ref{init:lemma:bound_power_vector_norm_with_incoherence}. 
Moreover, by applying $\BL=\BW_0\bm{\Lambda} \BW_0^T+\ddot{\BW}_0\ddot{\bm{\Lambda}}\ddot{\BW}_0^T -\BE$, we get
\begin{align*}
&~\left\|\BL\BW_0\bm{\Lambda} ^{-(a+b+1)}\BW_0^T\BL\right\|_2 \cr
=&~ \left\|(\BW_0\bm{\Lambda} \BW_0^T+\ddot{\BW}_0\ddot{\bm{\Lambda}}\ddot{\BW}_0^T -\BE)\BW_0\bm{\Lambda} ^{-(a+b+1)}\BW_0^T(\BW_0\bm{\Lambda} \BW_0^T\right. \cr
 &\quad+\left.\ddot{\BW}_0\ddot{\bm{\Lambda}}\ddot{\BW}_0^T -\BE)\right\|_2        \cr
=&~ \left\|\BW_0\bm{\Lambda} ^{-(a+b-1)}\BW_0^T-\BE\BL\BW_0\bm{\Lambda} ^{-(a+b)}\BW_0^T-\BL\BW_0\bm{\Lambda} ^{-(a+b)}\BW_0^T\BE \right. \cr
&\quad+\left. \BE\BL\BW_0\bm{\Lambda} ^{-(a+b+1)}\BW_0^T\BE\right\|_2 \cr
\leq&~ |\lambda_r|^{-(a+b-1)} + |\lambda_r|^{-(a+b)}\left\|\BE\right\|_2+ |\lambda_r|^{-(a+b)}\left\|\BE\right\|_2+ |\lambda_r|^{-(a+b+1)}\left\|\BE\right\|_2^2 \cr
=&~ |\lambda_r|^{-(a+b-1)}\left( 1+ \frac{2\|\BE\|_2}{|\lambda_r|}+\left(\frac{\|\BE\|_2}{|\lambda_r|}\right)^2 \right)  \cr
=&~ |\lambda_r|^{-(a+b-1)}\left( 1+ \frac{\left\|\BE\right\|_2}{|\lambda_r|} \right)^2  \cr
\leq&~ 2|\lambda_r|^{-(a+b-1)} \cr
\leq&~ 2\left(\frac{7}{8}\sigma_r(\BL)\right)^{-(a+b-1)}
\end{align*}
where the second inequality follows from \eqref{init:eq:eigenvalues bound 2} and the last inequality follows from \eqref{init:eq:eigenvalues bound 1}. Hence, 
\begin{align} \label{init:eq:sum Y_ab bound}
\sum_{a+b>0}\bm{Y}_{ab}&\leq  \sum_{a+b>0} 2\alpha \sqrt{\mu r n} \left\|\BE\right\|_\infty \left(\frac{\frac{1}{8}\sigma_r^L}{\frac{7}{8}\sigma_r^L }\right)^{a+b-1}  \cr
                  &\leq 2\alpha \sqrt{\mu r n} \left\|\BE\right\|_\infty  \sum_{a+b>0} \left(  \frac{1}{7}\right)^{a+b-1}\cr
                  &\leq 2\alpha \sqrt{\mu r n}\left\|\BE\right\|_\infty  \left(  \frac{1}{1-\frac{1}{7}}\right)^2\cr
                  &\leq 3\alpha \sqrt{\mu r n} \left\|\BE\right\|_\infty.  
\end{align}

Combining \eqref{init:eq:Y0 bound} and \eqref{init:eq:sum Y_ab bound}, we have the \textit{row version} of the inequality:
\begin{align*} \label{init:eq: L-L0 inf norm}
\max_i\left\|(\BL_0-\BL)^T\bm{e}_i\right\|_2 &\leq \bm{Y}_0 + \sum_{a+b>0} \bm{Y}_{ab}  \cr
                     &\leq 5\alpha \sqrt{\mu r n} \left\|\BE\right\|_\infty + 3\alpha \sqrt{\mu r n} \left\|\BE\right\|_\infty \cr 
                     &\leq \frac{32\alpha \mu^{1.5} r^{1.5}}{\sqrt{n}}\sigma_{\max}(\BL)
\end{align*}
where the last step uses \eqref{eq:SminusS}. 

The \textit{column version} of the inequality, i.e.,
\[
\max_j\left\|(\BL_0-\BL)\bm{e}_j\right\|_2 \leq \frac{32\alpha \mu^{1.5} r^{1.5}}{\sqrt{n}}\sigma_{\max}(\BL), 
\]
can be proved similarly. 
This finishes the proof.
\end{proof}

\end{document}